\documentclass[sigconf]{acmart}
\AtBeginDocument{%
  }

\copyrightyear{2026}
\acmYear{2026}
\setcopyright{cc}
\setcctype{by}
\acmConference[KDD '26]{Proceedings of the 32nd ACM SIGKDD Conference on Knowledge Discovery and Data Mining V.2}{August 09--13, 2026}{Jeju Island, Republic of Korea}
\acmBooktitle{Proceedings of the 32nd ACM SIGKDD Conference on Knowledge Discovery and Data Mining V.2 (KDD '26), August 09--13, 2026, Jeju Island, Republic of Korea}
\acmDOI{10.1145/3770855.3817664}
\acmISBN{979-8-4007-2259-2/2026/08}

\usepackage{xspace}
\usepackage{enumitem}
\usepackage[normalem]{ulem}
\usepackage{makecell}
\usepackage{hyperref}
\usepackage[table]{xcolor}
\usepackage[hang, flushmargin]{footmisc}
\usepackage{xcolor} % 只改字体颜色，不改背景
\newcommand{\best}[1]{\normalsize \textbf{#1}}
\usepackage{subcaption}
\usepackage{prompt_style}
\usepackage{balance}
\newcommand{\std}[2]{#1 \scalebox{0.8}{$\pm$ #2}}

\begin{document}

% \title{\sys: Collaborating LLM-based Feature Engineering and Bayesian Hyperparameter Optimization for Robust AutoML}

\title{\sys: LLM-driven Feature Engineering Empowered by Collaborative Bayesian Hyperparameter Optimization}

\author{Beicheng Xu}
\authornote{Both authors contributed equally to this research.}
\affiliation{%
  \institution{School of CS \& Key Lab of High Confidence Software Technologies (MOE), Peking University}
  \city{Beijing}
  \country{China}
}
\email{beichengxu@stu.pku.edu.cn}

\author{Keyao Ding}
\authornotemark[1]
\affiliation{%
  \institution{School of CS \& Key Lab of High Confidence Software Technologies (MOE), Peking University}
  \city{Beijing}
  \country{China}
}
\email{maodeshi@stu.pku.edu.cn}

\author{Wei Liu}
\affiliation{%
  \institution{School of CS \& Key Lab of High Confidence Software Technologies (MOE), Peking University}
  \city{Beijing}
  \country{China}
}
\email{eularioal@stu.pku.edu.cn}

\author{Yupeng Lu}
\affiliation{%
  \institution{School of CS \& Key Lab of High Confidence Software Technologies (MOE), Peking University}
  \city{Beijing}
  \country{China}
}
\email{xinkelyp@pku.edu.cn}

\author{Bin Cui}
% \authornote{Corresponding author.}
\affiliation{%
  \institution{School of CS \& Beijing Key Laboratory of Software and Hardware Cooperative Artificial Intelligence Systems, Peking University}
  \city{Beijing}
  \country{China}
}
% \affiliation{%
%   \institution{}
%   \city{Beijing}
%   % \state{Ohio}
%   \country{China}
% }
\email{bin.cui@pku.edu.cn}

\newcommand{\sys}{CoFEH\xspace}

\begin{abstract}
Feature Engineering (FE) is pivotal in automated machine learning (AutoML) but remains a bottleneck for traditional methods, which operate within rigid search spaces and lack domain awareness.
While Large Language Models (LLMs) offer a promising alternative to generate unbounded operators with semantic reasoning, existing methods focus on isolated subtasks such as feature generation, falling short of free-form FE pipelines. 
Moreover, they are rarely coupled with hyperparameter optimization (HPO) of the downstream ML model, leading to greedy “FE-then-HPO” workflows that cannot capture strong FE–HPO interactions.
In this paper, we present \sys, a collaborative framework that interleaves LLM-based FE and Bayesian HPO for robust end-to-end AutoML.
CoFEH uses an LLM-driven FE optimizer powered by Tree of Thought (TOT) to explore flexible FE pipelines, a Bayesian optimization (BO) module to solve HPO, 
and a dynamic optimizer selector that adaptively interleaves FE and HPO steps.
Crucially, we introduce a mutual conditioning mechanism that shares context between LLM and BO, enabling mutually informed decisions.
% Experiments show that \sys not only outperforms traditional and LLM-based FE baselines, but also achieves superior end-to-end performance under joint optimization.
Experiments show that \sys outperforms both traditional and LLM-based baselines in both standalone FE and joint FE+HPO settings.
% Our implementation is available at \url{https://github.com/PKU-DAIR/cofeh}.
% An extended version with full appendices is available at \url{https://arxiv.org/abs/2602.09851}.
\end{abstract}



%   \includegraphics[width=\textwidth]{images/LLMFE_workflow.pdf}
%   \caption{Comparison of optimization workflows: existing methods vs. \sys.}
%   \Description{Optimization Pipeline of Existing Methods and \sys.}


\maketitle
\newcommand\kddavailabilityurl{https://doi.org/xxxx}
\ifdefempty{\kddavailabilityurl}{}{
\begingroup\small\noindent\raggedright\textbf{Resource Availability:}\\
% please change the following context to include multiple artifacts if necessary, including data, models, code, etc.
The source code is available at \url{https://doi.org/10.5281/zenodo.20323800}.
Extended version with full appendices: \url{https://arxiv.org/pdf/2602.09851}.
\endgroup
}

\section{Introduction}

\begin{figure*}
  \includegraphics[width=\textwidth]{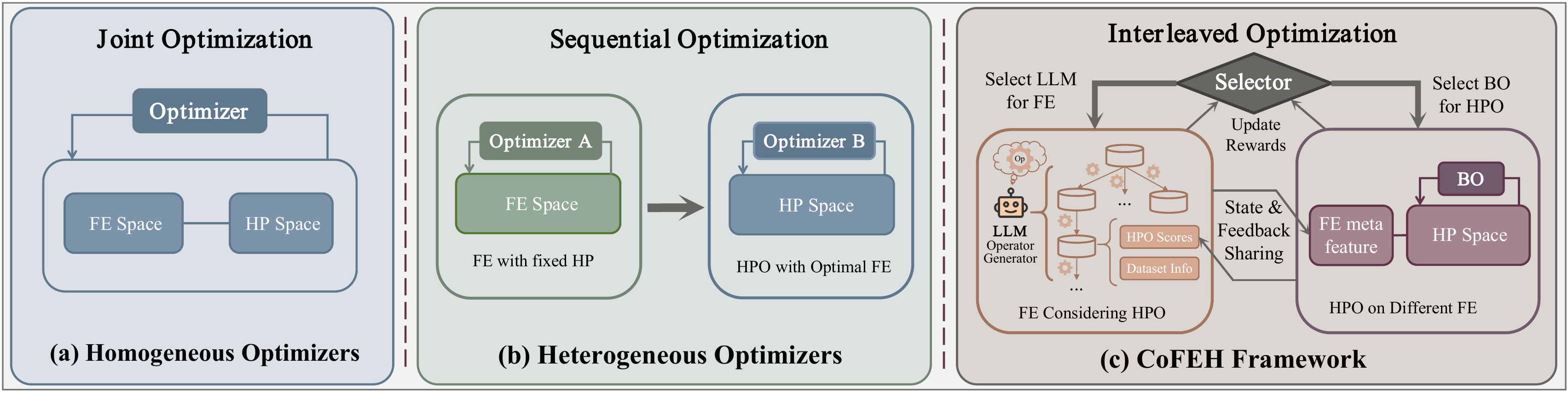}
  \caption{Comparison of optimization workflows: existing methods vs. \sys.}
  \Description{Optimization Pipeline of Existing Methods and \sys.}
  \label{fig:teaser}
% \vspace{-0.5em}
\end{figure*}

The success of machine learning (ML) hinges on the synergy between data representation and model capacity~\cite{mlsurvey_jordan2015machine}. 
Therefore, Feature Engineering (FE) serves as the cornerstone of an ML pipeline, directly influencing ML models’ performance and effectiveness.
Broadly defined, FE constitutes a holistic transformation process that bridges raw data and model inputs, encompassing preprocessing (e.g., imputation, normalization), transformation (e.g., encoding, discretization), generation (synthesizing new interactions), and selection (filtering irrelevant signals)~\cite{fesurvey_mumuni2025automated,zoppi2025strategy}.
However, manual FE is labor-intensive and demands deep domain expertise.
To lower barriers and streamline deployment, the AutoML community has proposed several methods to automate the FE process~\cite{automlsurvey_yao2018taking,openfe_zhang2023openfe,fe4_horn2019autofeat,fe1_katz2016explorekit,fe2_kaul2017autolearn,fe3_khurana2018feature}.
Many end-to-end AutoML systems also incorporate FE by casting it as a search problem over a predefined space of transformation operations and pipeline templates~\cite{autoweka_thornton2013,auto-sklearn_feurer2022auto,rb_li2020efficient,h2o_ledell2020h2o}.
Leveraging an optimizer like Bayesian optimization (BO), they explore finite configurations to construct FE pipelines.

Despite facilitating automation, these traditional approaches come with three fundamental drawbacks: 
(i) They are semantics-agnostic and depend on expensive trial-and-error without leveraging domain priors; 
(ii) They enforce rigid FE pipeline templates, often restricting pipelines to fixed sequences (e.g., generation followed by selection~\cite{fe4_horn2019autofeat,openfe_zhang2023openfe}) or isolated sub-tasks, precluding dynamic interleaving or flexible composition; 
(iii) They restrict the search to closed operation libraries (e.g., a small set of built-in preprocessing algorithms or basic arithmetic), which prevents discovering task-specific operations beyond predefined primitives.

To address these gaps, Large Language Models (LLMs) have been introduced to automated FE~\cite{ellm_gong2025evolutionary,lmpriors_choi2022lmpriors,llmfe-preprocessor1_wang2025data}.
Through domain knowledge and code generation capabilities~\cite{chen2025coderankeval}, LLMs can propose semantically grounded operations and bespoke logic beyond fixed operation libraries, alleviating drawbacks (i) and (iii) mentioned above.
However, most existing methods target only a single FE subtask—most notably feature generation~\cite{caafe_hollmann2023large,octree_nam2024optimized,ellm_gong2025evolutionary,lfg_10.24963/ijcai.2025/782,llmfe_abhyankar2025llm}—resulting in a homogeneous and simplistic FE pipeline. This leads to our first challenge:
% \emph{\textbf{C1: Free-form FE.} How can we empower LLMs to construct a truly free-form, human-expert-like FE pipeline that flexibly composes heterogeneous operations (preprocessing, transformation, generation, and selection) rather than only homogeneous ones?}
\emph{\textbf{C1: Free-form FE.} How can we empower LLMs to construct a truly free-form FE pipeline that flexibly composes heterogeneous operations (preprocessing, transformation, generation, and selection)?}

While LLMs are inherently well-suited for FE, existing methods overlook the critical synergy with Hyperparameter Optimization (HPO) of the downstream ML model. 
In modern AutoML, HPO is predominantly driven by Bayesian optimization (BO) due to its sample efficiency and uncertainty modeling~\cite{divbo_shen2022divbo,optdivbo_poduval2024cash,rb_li2020efficient}. 
Consequently, most LLM-based FE methods that aim for end-to-end performance must attach a BO-based HPO module. 
Because the two optimizers are heterogeneous and operate over differently represented spaces, they typically fall back to a greedy sequential routine—first optimizing FE under a fixed model, then tuning hyperparameters on the frozen features (Fig.~\ref{fig:teaser}(b))~\cite{caafe_hollmann2023large,octree_nam2024optimized,ellm_gong2025evolutionary,autokaggle_li2024autokaggle}.
This greedy strategy ignores the strong dependency between feature representation and model capacity. 
In contrast, traditional methods (e.g., Auto-sklearn~\cite{auto-sklearn_feurer2022auto}) naturally enable joint optimization by enforcing a homogeneous search space (Fig.~\ref{fig:teaser}(a)). 
This creates a ``capability-integration paradox'': LLMs excel at FE exploration but lag behind traditional methods in FE--HPO coupling.
To couple LLM-based FE with BO-based HPO, we advocate an \textbf{interleaved} optimization, which raises two further challenges:
\emph{\textbf{C2: Collaborative FE--HPO optimization.} How can we jointly optimize LLM-based FE and HPO in a coupled manner, rather than optimizing each in isolation?}
and
\emph{\textbf{C3: Task-adaptive scheduling.} How can we allocate budget between FE and HPO according to their task-specific utility?}

To address these challenges, we propose CoFEH (Fig.~\ref{fig:teaser}(c)), a novel framework for holistic ML pipeline optimization. CoFEH is designed to unleash the full potential of LLMs in crafting free-form, human-expert-like FE pipelines, while coupling this with a state-of-the-art BO–based HPO. Our contributions are as follows:
(i) For \textbf{C1}, we propose an LLM-driven tree-of-thought FE pipeline optimizer to explore unconstrained pipeline topologies and operators, achieving truly free-form feature engineering.
(ii)  For \textbf{C2}, we introduce a mutual conditioning mechanism, which establishes a bidirectional information flow between the LLM and BO, enabling them to make decisions conditioned on each other to avoid isolated, sub-optimal tuning.
(iii) For \textbf{C3}, we introduce a PUCB-based dynamic optimizer selector that adaptively allocates the optimization budget between FE and HPO, facilitating efficient interleaved optimization.
(iv) Experiments on 28 public datasets show that CoFEH outperforms state-of-the-art traditional and LLM-based baselines in both standalone FE and end-to-end pipeline optimization.
% \begin{itemize}[leftmargin=2em, topsep=5pt, partopsep=5pt, itemsep=3pt, parsep=0pt]
%  \item For \textbf{C1}, we propose an LLM-driven tree-of-thought FE pipeline optimizer to explore unconstrained pipeline topologies and operations, achieving truly free-form feature engineering.
%  \item  For \textbf{C2}, we introduce a mutual conditioning mechanism, which establishes a bidirectional information flow between the LLM and BO, enabling them to make decisions conditioned on each other to avoid isolated, sub-optimal tuning.
%  \item For \textbf{C3}, we introduce a PUCB-based dynamic optimizer selector that adaptively allocates the optimization budget between FE and HPO, facilitating efficient interleaved optimization.
%  \item Experiments on 28 public datasets show that CoFEH outperforms state-of-the-art traditional and LLM-based baselines in both standalone FE and end-to-end pipeline optimization.

% \end{itemize}
\section{Background and Motivation}

\subsection{The Formal Machine Learning Pipeline}

The standard supervised machine learning pipeline can be formalized as a hierarchical optimization problem. Given a dataset $\mathcal{D} = \{(\mathbf{x}_i, y_i)\}_{i=1}^N$ where $\mathbf{x}_i \in \mathcal{X}$ represents the raw input space and $y_i \in \mathcal{Y}$ the target, the objective is to learn a mapping from raw data to target that minimizes the generalization error. This process is typically decomposed into two interconnected sub-problems: \textbf{Feature Engineering (FE)}~\cite{fe1_katz2016explorekit,fe2_kaul2017autolearn,fe3_khurana2018feature,fe4_horn2019autofeat,li2020you} of data and \textbf{Hyperparameter Optimization (HPO)}~\cite{tpe_bergstra2011algorithms,bohb_falkner2018bohb,kallab2024towards,hpo2_hu2019multi} of ML model. 

To automate pipeline construction, traditional AutoML systems cast the entire process into a black-box optimization framework, and leverage method like BO~\cite{autoweka_thornton2013,auto-sklearn_feurer2022auto,rb_li2020efficient,xu2026pseo}, genetic programming~\cite{tpot_olson2016tpot}, reinforcement learning~\cite{AlphaD3M_drori2021alphad3m}, and ADMM~\cite{admm_liu2020admm} to navigate the search space.
% Among them, Bayesian Optimization (BO)~\cite{gp_snoek2012practical,smac_hutter2011sequential} has emerged as the dominant paradigm, serving as the cornerstone for leading systems such as Auto-WEKA~\cite{autoweka_thornton2013}, Auto-sklearn~\cite{auto-sklearn_feurer2022auto}, MindWare~\cite{rb_li2020efficient}, and H2O~\cite{h2o_ledell2020h2o}.
% In addition, genetic programming~\cite{tpot_olson2016tpot}, reinforcement learning~\cite{AlphaD3M_drori2021alphad3m}, and ADMM-based methods~\cite{admm_liu2020admm} are also used to optimize the pipeline.
Complementing the traditional approaches, Large Language Models (LLMs) have introduced a knowledge-driven paradigm, acting as expert-like agents that propose task-adaptive feature transformation and model configurations~\cite{text2ml_xu2024large,caafe_hollmann2023large,agenthpo_liu2024large,autokaggle_li2024autokaggle,mlmaster_liu2025ml,autom3l_luo2024autom3l}. 
% This evolution enables a transition from purely statistical black-box optimization to knowledge-informed proposal generation.
% Recent systems such as AutoKaggle~\cite{autokaggle_li2024autokaggle}, ML-Master~\cite{mlmaster_liu2025ml}, and AutoM$^3$L~\cite{autom3l_luo2024autom3l} explore this direction by using LLMs as optimizers or controllers to generate and refine pipelines, shifting from purely statistical black-box optimization to knowledge-informed proposal generation.

In the remainder of this section, we analyze prior work from the perspectives of FE and HPO, and discuss the key challenges and opportunities in jointly optimizing them within a unified framework.

\subsection{Feature Engineering}
\label{sec:background_fe}
% We define FE broadly to encompass the entire transformation pipeline from raw data to the model input space. 
We define FE in a broad sense, covering the entire process from raw data to the model input space.
Let $\mathcal{T}$  denote the space of typed \textbf{feature operations}, encompassing data preprocessing, transformations, feature generation, and selection, etc. 
% Let $\mathcal{T}$  denote the space of typed feature operators, encompassing data preprocessing (e.g., imputation and scaling), representation transformations (e.g., encoding and discretization), feature generation (e.g., non-linear transforms and interactions), and feature selection (e.g., importance ranking, statistical significance), etc. 
Consequently, an FE pipeline can be expressed as a composition of $k$ operations,
\begin{equation}
\label{eq:fe_pipeline}
    T = t_k \circ t_{k-1} \circ \dots \circ t_1,\quad t_j \in \mathcal{T}.
\end{equation}
The ultimate goal of this FE pipeline is to project the raw input space $\mathcal{X}$ into an optimized latent space $\mathcal{X}' = T(\mathcal{X})$ for the downstream ML model. 
Ideally, feature engineering is a creative and semantics-intensive process driven by domain expertise. As illustrated in the bottom panel of Fig.~\ref{fig:fe_pipeline}, human experts do not adhere to a rigid template. Instead, they leverage prior knowledge and intuition to navigate an unbounded search space to build a free-form pipeline.

However, such an unbounded process remains beyond the reach of traditional optimization algorithms (e.g., BO), which necessitate well-defined, finite search spaces. Consequently, existing AutoML systems typically impose severe structural constraints to make the problem tractable~\cite{tpot_olson2016tpot,auto-sklearn_feurer2022auto,autoweka_thornton2013}. As depicted in the upper panels of Fig.~\ref{fig:fe_pipeline} (e.g., MindWare~\cite{rb_li2020efficient}, OpenFE~\cite{openfe_zhang2023openfe}), these simplifications result in three limitations:
(i)  \textbf{Semantic agnosticism:} Lacking the capacity to incorporate domain knowledge, these systems rely exclusively on brute-force, data-driven trial-and-error.
(ii) \textbf{Rigid topology:} The pipeline structures are ossified into limited and fixed sequences, preventing dynamic re-ordering or recursion. For instance, MindWare enforces a strict four-stage workflow, while OpenFE is confined to a generation-selection pipeline.
(iii) \textbf{Restricted operation set:} The search is confined to a closed library of predefined algorithms or mathematical primitives (e.g., $+,-,\times$), precluding novel operations beyond the system's hard-coded scope.

\begin{figure}[t!]
  \centering
  \includegraphics[width=\linewidth]{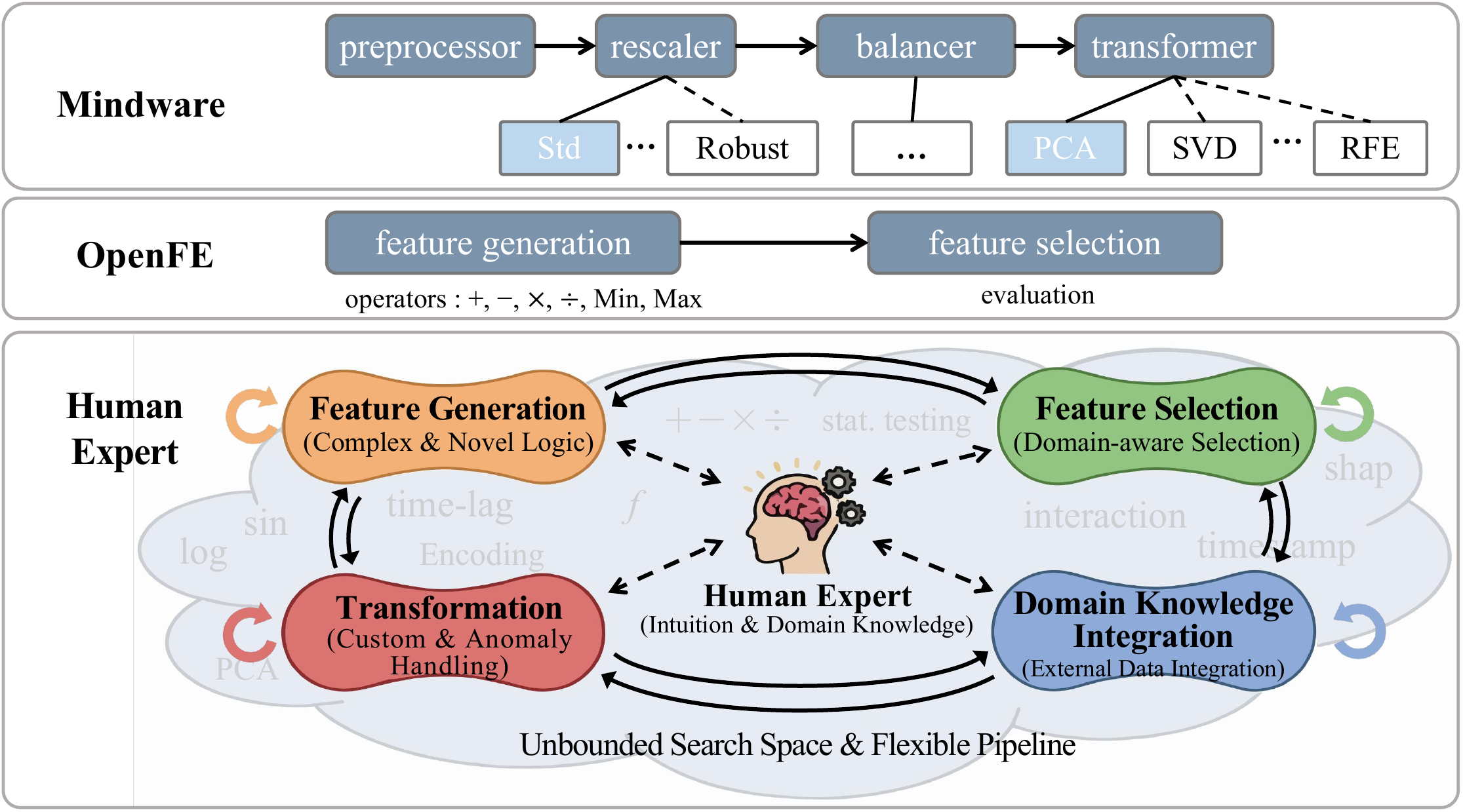}
  \caption{AutoML vs. human expert: The freedom of FE.}
  \label{fig:fe_pipeline}
% \vspace{-1em}
\end{figure}

To overcome these bottlenecks, recent advances have proposed LLM-based FE as an alternative. 
% With extensive domain knowledge and generation capabilities, LLMs can emulate human-like semantic reasoning, enabling the construction of flexible FE pipelines with unbounded operations~\cite{llm4automlsurvey_gu2025large}. 
This direction is part of a broader effort to apply LLMs to tabular-data tasks~\cite{dong2024use}.
With extensive domain knowledge and code-generation capabilities, LLMs can emulate human-like semantic reasoning, enabling flexible feature transformations beyond closed operation libraries.
This potential has sparked a wave of research across the FE spectrum, including preprocessing~\cite{llmfe-preprocessor1_wang2025data,llmfe-preprocessor2_choi2023aliro}, feature selection~\cite{lmpriors_choi2022lmpriors,llmselect_jeong2024llm}, and feature generation~\cite{caafe_hollmann2023large,octree_nam2024optimized,ellm_gong2025evolutionary,lfg_10.24963/ijcai.2025/782,llmfe_abhyankar2025llm,ko2025ferg}.
However, these studies typically focus on isolated sub-tasks, most notably feature generation. Consequently, they yield monolithic pipelines restricted to homogeneous operations, precluding the dynamic interleaving of diverse FE operations.

% \textbf{\em Conclusion \#1: while LLMs can leverage semantic priors to explore an unbounded operator space for FE, realizing their full potential requires moving toward truly free-form, end-to-end pipeline topologies.}
\textbf{\em Conclusion \#1: The intrinsically semantic nature of FE makes LLMs its natural architect, yet fully harnessing this power requires adopting truly free-form pipeline topologies.}

\subsection{Hyperparameter Optimization}
Given the transformed features $\mathcal{X}'$, a learner $\mathcal{A}$ is selected to induce a model. In AutoML, this is formulated as the Combined Algorithm Selection and Hyperparameter Optimization (CASH) problem~\cite{autoweka_thornton2013}
% We treat the choice of the algorithm (e.g., XGBoost, MLP) as a top-level categorical hyperparameter $\lambda_{\text{alg}}$, with a corresponding set of conditional hyperparameters $\boldsymbol{\lambda}_{\text{sub}}$ (e.g., learning rate). The full configuration is $\boldsymbol{\lambda} = (\lambda_{\text{alg}}, \boldsymbol{\lambda}_{\text{sub}})$, and HPO aims to find $\boldsymbol{\lambda}^*$ that minimizes the generalization error.
In this work, we use HPO as an umbrella term that includes CASH, where the configuration space $\Lambda$ jointly encompasses both the choice of the learning algorithm (e.g., XGBoost, MLP) and its associated hyperparameters. The goal of HPO is to identify the optimal configuration $\boldsymbol{\lambda}^*$ that minimizes the generalization error.

The prevailing approach for solving this task is Bayesian Optimization (BO)~\cite{autoweka_thornton2013,auto-sklearn_feurer2022auto,rb_li2020efficient,h2o_ledell2020h2o,lightautoml_vakhrushev2021lightautoml,divbo_shen2022divbo,optdivbo_poduval2024cash}. 
BO solves black-box problems by iterating three steps~\cite{tpe_bergstra2011algorithms,smac_hutter2011sequential,gp_snoek2012practical}:
1) fit a surrogate $f$ on the observed data $\mathcal{O} = \{(\boldsymbol{\lambda}_i, v_i)\}_{i=1}^{n-1}$;
2) sample and select the next configuration by maximizing an acquisition function $\alpha$: $\boldsymbol{\lambda}_{n}=\arg\max_{\boldsymbol{\lambda}}\alpha(\boldsymbol{\lambda}; f)$;
3) evaluate $\boldsymbol{\lambda}_n$ to obtain the performance $v_n$ and update $\mathcal{O} \leftarrow \mathcal{O} \cup \{(\boldsymbol{\lambda}_n, v_n)\}$.
BO is favored for its principled modeling of complex search spaces and acquisition-driven balance of exploration and exploitation.
While recent works have tried to adapt LLMs for HPO~\cite{agenthpo_liu2024large,automl-gpt_zhang2023automl,llambo_liularge,bora_10.24963/ijcai.2025/553}, LLM optimizers suffer from inherent limitations, lag behind BO
methods, and prove unreliable in many practical scenarios~\cite{huang2024exploring}. They lack a surrogate model for objective and uncertainty quantification, and struggle to utilize the full optimization history due to context window limits.

% \textbf{\em Conclusion \#2: BO remains the gold standard for HPO with mathematical guarantees and modeling capabilities, surpassing the current reach of LLMs in hyperparameter spaces.}
\textbf{\em Conclusion \#2: BO remains the gold standard for HPO, surpassing the current reach of LLMs in hyperparameter spaces.}

\subsection{Joint Optimization of FE and HPO}
\label{sec:background_joint_opt}

While FE and HPO are often treated as orthogonal tasks, they are in fact strongly coupled.
We formulate AutoML as an optimization problem that simultaneously searches for the optimal FE pipeline $T^*$ and model configuration $\boldsymbol{\lambda}^*$ to minimize the validation loss:
\begin{equation}
    (T^*, \boldsymbol{\lambda}^*) = \mathop{\arg\min}_{T \in \mathcal{T}^k, \boldsymbol{\lambda} \in \Lambda} \mathcal{L}_{val}\Big(\mathcal{A}(T(\mathcal{D}_{train}); \boldsymbol{\lambda}), T(\mathcal{D}_{val}) \Big),
\end{equation}
where $\mathcal{T}^k$ represents the unbounded space of feature engineering pipelines and $\Lambda$ denotes the hyperparameter space of learner $\mathcal{A}$.

Existing approaches to this problem can be categorized based on the nature of their optimizers, as visualized in Fig.~\ref{fig:teaser}.
(i) \textbf{Homogeneous optimizer framework} (Fig.~\ref{fig:teaser}(a)), employs a single type of optimizer for both tasks, which naturally enables a joint optimization.
Traditional AutoML systems, such as TPOT~\cite{tpot_olson2016tpot}, Auto-sklearn~\cite{auto-sklearn_feurer2022auto}, and Mindware~\cite{rb_li2020efficient}, integrate FE and HPO into a unified yet restricted search space, enabling joint modeling and optimization.
Similarly, LLM-based frameworks like ML-Master~\cite{mlmaster_liu2025ml} and AIDE~\cite{AIDE_jiang2025aide} treat ML pipeline generation as a unified code synthesis task, prompting LLMs to co-propose feature transformations and model configurations within a single reasoning context.
(ii) \textbf{Heterogeneous optimizer framework} (Fig.~\ref{fig:teaser}(b)), adopts distinct optimization strategies for FE and HPO.
This scenario typically necessitates a sequential, greedy strategy: optimizing FE with a fixed learner, followed by HPO on the derived feature set.
Traditional methods (e.g., AutoFeat~\cite{fe4_horn2019autofeat}, OpenFE~\cite{openfe_zhang2023openfe}) optimize features via fixed proxy models before performing post-hoc HPO.
Similarly, LLM-based frameworks such as CAAFE~\cite{caafe_hollmann2023large} and ELLM-FT~\cite{ellm_gong2025evolutionary} focus exclusively on feature evolution, treating HPO as an isolated downstream task.
% Even end-to-end systems like AutoKaggle~\cite{autokaggle_li2024autokaggle} often generate FE code via LLM and then attach a separate HPO (e.g., via grid search).
Crucially, this sequential approach suffers from \textbf{greedy myopia}: it fails to capture the strong dependency between features and model hyperparameters. By fixing hyperparameters during FE, the optimizer inevitably discards potential high-value features that require specific hyperparameter tuning to perform well, thereby trapping the system in sub-optimal local minima.

Therefore, although Conclusions \#1 and \#2 suggest that LLMs are better suited for FE while BO is more effective for HPO, jointly optimizing them is non-trivial due to their inherent heterogeneity. 
Instead of a greedy sequential optimization, we argue for an \textbf{interleaved optimization} strategy that continuously evaluates different (FE pipeline, model configuration) combinations to avoid poor local optima. 
This raises two key challenges: 
(i) \textbf{mutual conditioning and shared context}—the LLM must refine FE with awareness of how features interact with different configurations, and BO must model hyperparameters conditioned on the FE pipeline. 
Otherwise, if the LLM evaluates feature quality without distinguishing the underlying configurations (and vice versa for BO), the historical performance data becomes noisy and misleading, obscuring the true causal impact of any single modification. 
(ii) \textbf{adaptive resource allocation}—the system must decide when to invest budget in improving FE versus HPO, since their marginal utility varies substantially across tasks (e.g., some datasets are feature-sensitive while HPO yields limited gains, and vice versa).

\textbf{\em Conclusion \#3: Effective AutoML requires interleaving LLM-based FE and BO-based HPO, with mutual conditioning and task-adaptive budget allocation.}
\section{Method}
\label{sec:method}

In this section, we present \sys, a framework for LLM-driven \uline{\textbf{F}}eature \uline{\textbf{E}}ngineering empowered by \uline{\textbf{Co}}llaborative \uline{\textbf{H}}yperparameter optimization. 
It comprises three core components: (i) an LLM-based Tree-of-Thought search for FE pipeline; (ii) a mutual condition mechanism that enables collaborative tuning between the LLM-driven FE and BO-based HPO; and (iii) a dynamic optimizer selector to adaptively allocate resources across the optimization process.

\subsection{LLM-based Feature Engineering}
\label{sec:llm_fe}

To effectively navigate the unbounded and semantics-intensive search space of feature engineering, we reformulate the construction of the pipeline from Eq.~\eqref{eq:fe_pipeline} as a \textbf{sequential decision-making problem}.
We define a \textbf{dataset state} $s$ as the representation of the data at a given stage of transformation, where the initial state $s_0 = \mathcal{X}$ is the raw dataset. Each operation $t_j \in \mathcal{T}$ is viewed as an action that transitions the system from state $s_{j-1}$ to a new state $s_j = t_j(s_{j-1})$. Consequently, the optimization of the FE pipeline is equivalent to finding an optimal sequence of \textbf{actions} $(t_1, t_2, \dots, t_k)$ that maximizes the downstream performance reward:
\begin{equation}
    \max_{t_1, \dots, t_k} \text{Reward}(t_k \circ t_{k-1} \circ \dots \circ t_1(s_0)).
\end{equation}
Standard linear prompting is ill-suited for this complex search space as it follows a singular reasoning path without the capacity for backtracking. 
To bridge this gap, we adopt the Tree of Thought (ToT) paradigm~\cite{tot_yao2023tree}, enabling the LLM to explore multiple reasoning branches.
As illustrated in the search tree (Fig.~\ref{fig:CoFEH_FE_workflow}, left), intermediate dataset states ($s_i$) are treated as ``thoughts''—discrete nodes in a hierarchical search space connected by LLM-generated FE operations.
We implement this ToT search via a Monte Carlo Tree Search (MCTS)~\cite{mcts_kocsis2006bandit}, which systematically balances the discovery of novel operations with the exploitation of historical knowledge through four iterative steps (Fig.~\ref{fig:CoFEH_FE_workflow}):
1) \textbf{Selection}: MCTS traverses from the root ($s_0$) along a selection path to identify the most promising dataset state for further expansion.
2) \textbf{Expansion}: the LLM is prompted to generate a new FE operation to process the dataset in the selected node.
3) \textbf{Playout}: Utilizing the executable code synthesized by the LLM, the dataset is transformed and subsequently evaluated through a downstream ML model. This process returns a validation score $v_i$, thereby instantiating a new dataset state $s_{i+1}$ as a successor node in the tree.
4) \textbf{Backpropagation}: The results are propagated back to update three statistics for each ancestor node $s$: visit count $N_s$, cumulative reward $R_s$, and subtree best performance $v^{\max}_s$ (the maximum score observed within its subtree). We first compute a binary reward $r = \mathbb{I}(v_{new} > v^{\max}_{s_{0}})$, indicating whether a new global optimum has been achieved. Subsequently, we update the statistics: $N_s \leftarrow N_s + 1$, $R_s \leftarrow R_s + r$, and $v^{\max}_s \leftarrow \max(v^{\max}_s, v_{new})$.

\begin{figure}[t!]
  \centering
  \includegraphics[width=\linewidth]{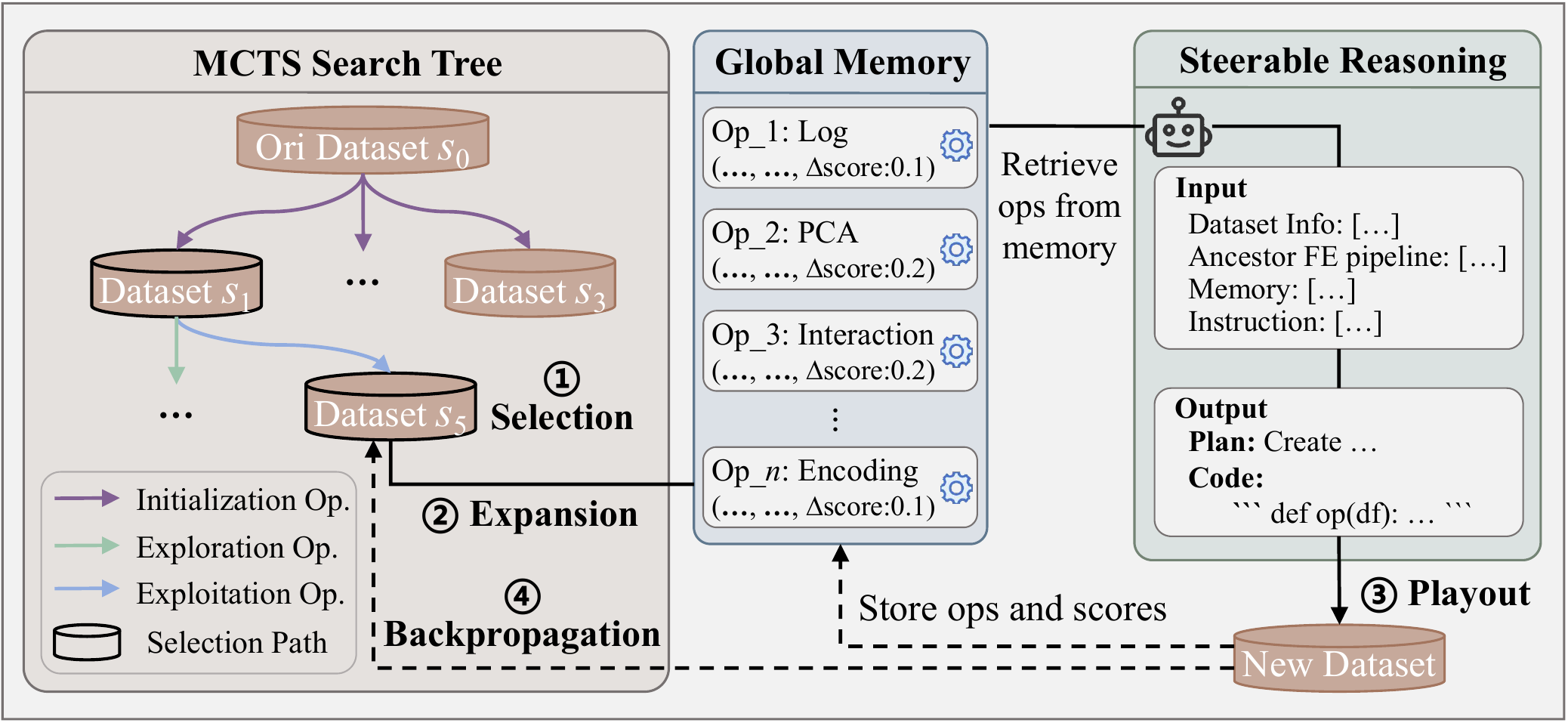}
  \caption{FE workflow of \sys.}
  \label{fig:CoFEH_FE_workflow}
% \vspace{-1.5em}
\end{figure}

\subsubsection{Selection Down the MCTS Tree}

The selection phase initiates at the root node $s_0$, representing the original dataset. 
It traverses the search tree by recursively choosing the child node that maximizes the Upper Confidence Bound for Trees (UCT) criterion~\cite{uct_auer2002using}:
\begin{equation}
\label{eq:uct_selection}
s^* = \arg\max_{s' \in \text{children}(s)} \left( Q(s') + C_1 \cdot \sqrt{\frac{\ln N_s}{N_{s'}}} \right),
\end{equation}
where $C_1$ is a hyperparameter balancing exploitation and exploration. The exploitation term $Q(s')$ evaluates the potential quality of the candidate feature transformation path and is defined as:
\begin{equation}
\label{eq:puct_qvalue}
Q(s') = \frac{R_{s'}}{N_{s'}} + \tilde{v}^{\max}_{s'},
\end{equation}
where the first term $R_{s'}/N_{s'}$ represents the average reward, reflecting the historical success rate of discovering superior states along this branch. 
The second term $\tilde{v}^{\max}_{s'}$ denotes the normalized best validation performance observed within the subtree rooted at $s'$ (min–max normalized using the global
range of observed metric scores).
This formulation of $Q(s')$ ensures that the search prioritizes paths exhibiting both high absolute performance (via $\tilde{v}^{\max}$) and strong potential for iterative improvement (via the average reward). Through this mechanism, the algorithm identifies a selection path extending from $s_0$ to a non-fully expanded target node, as shown by the bold-bordered nodes in the left panel of
  Fig.~\ref{fig:CoFEH_FE_workflow}.

\subsubsection{Expansion Through Steerable Reasoning}

Upon identifying a target node $s_{\text{base}}$ through UCT policy, the steerable reasoning module expands the search tree by synthesizing a new FE operation. We construct a structured prompt for the LLM comprising the following core components:
(i) Dataset info ($\psi$): Semantic descriptions and statistical metadata of the base dataset to provide domain context.
(ii) Ancestor FE pipeline ($T_{\text{anc}}$): The sequence of operations previously applied from $s_0$ to $s_{\text{base}}$, which defines the current data state.
(iii) Memory ($\mathcal{M}_{s_\text{base}}$): A collection of high-performing operations archived from previous successful trials. 
Crucially, this component is provided only under the \textsc{Exploitation} directive to facilitate refinement.
(iv) Directive ($d$): An optimization instruction $d \in \{\textsc{Initialization}, \textsc{Exploration}, \textsc{Exploitation}\}$ that explicitly steers the LLM’s strategy.
Formally, the LLM generates a reasoning chain $\mathcal{R}$ followed by the code for a new operation $t_{\text{new}}$:
\begin{equation}
\label{eq:llm_reasoning}
(\mathcal{R}, t_{\text{new}}) = \text{LLM}\big(\psi, T_{\text{anc}}, \mathcal{M}_{s_\text{base}}, d\big).
\end{equation}
The synthesized operation code is then executed to transform the base dataset, instantiating a new dataset state:
$s_{\text{new}} = t_{\text{new}}(s_{\text{base}})$.
Crucially, the directive $d$ explicitly governs the search behavior:
\begin{itemize}[leftmargin=2em, topsep=3pt, partopsep=5pt, itemsep=2pt, parsep=0pt]
    \item $\textbf{\textsc{Initialization}}$ is invoked specifically for the root node $s_0$ to generate a high-quality initial FE operation.
    \item $\textbf{\textsc{Exploration}}$ encourages the LLM to propose novel operations to probe unknown regions of the search space.
    \item $\textbf{\textsc{Exploitation}}$ distills elite experiences and operations from $\mathcal{M}_{\text{global}}$ to refine the current dataset.
\end{itemize}
% (i) $\textbf{\textsc{Initialization}}$: is invoked specifically for the root node $s_0$ to generate a high-quality initial FE operation,
% (ii) $\textbf{\textsc{Exploration}}$: encourages the LLM to propose novel operations to probe unknown regions of the search space,
% and (iii) $\textbf{\textsc{Exploitation}}$: distills elite experiences and operations from $\mathcal{M}_{\text{global}}$ to refine the current dataset.

Detailed prompt template are shown in Appendix~\ref{app:prompt_design}. 
% A node is regarded as not fully expanded and remains eligible for selection until it satisfies a predefined expansion quota. 
The root node $s_0$ is treated as non-fully expanded until the $\textsc{Init.}$ directive has been executed five times. 
Any other node is considered fully expanded only after both $\textsc{Exploration}$ and $\textsc{Exploitation}$ directives have been performed twice. 
This ensures that each promising dataset state is thoroughly investigated through both creative discovery and historical refinement before the search moves deeper.

\subsubsection{Global Memory and Operation Retrieval}
As illustrated in the central module of Fig.~\ref{fig:CoFEH_FE_workflow}, every historical operation $t$ generated during the search is archived in the Global Memory $\mathcal{M}_{\text{global}}$. Each entry is stored as a tuple $(\mathcal{R}, \mathcal{F}_t, v, \Delta v)$, which includes: (i) the reasoning chain $\mathcal{R}$, (ii) the subset of features $\mathcal{F}_t$ required for the operation, which is parsed from the structured format of $\mathcal{R}$, (iii) the resulting performance score $v$, and (iv) the relative improvement $\Delta v = v - v_{\text{base}}$ achieved over its parent node.

To support the $\textsc{Exploitation}$ directive, we employ a two-stage retrieval mechanism to identify high-utility operations:
1) \textbf{Functional filtering}: We first ensure semantic compatibility by matching the required feature set $\mathcal{F}_t$ with the features currently available in state $S_{\text{base}}$. Only operations whose functional dependencies are fully satisfied by the current dataset are retained as valid candidates.
2) \textbf{Pareto-based selection}: From the filtered candidates, we identify high-quality operations by balancing absolute performance $v$ and relative gain $\Delta v$. We prioritize operations with high $\Delta v$ for their strong transformative potential, as well as those with high $v$, which represent the ability to refine performance even when the pipeline has reached well-performing regimes. Formally, we construct a Pareto frontier in the $(v, \Delta v)$ objective space and select the non-dominated operations to serve as ``elite experiences'' $\mathcal{M}_{s_\text{base}}$ within the steerable reasoning prompt.
This dual-metric selection ensures the LLM distills insights from both breakthrough transformations and marginal refinements recorded in the search history.

\subsection{Collaborative Tuning with HPO}

The final efficacy of an FE pipeline $T$ is inextricably linked to the hyperparameters $\boldsymbol{\lambda}$ of the downstream ML model. A decoupled approach, where FE is optimized independently of HPO, often results in the greedy myopia discussed in Section~\ref{sec:background_joint_opt}—the premature rejection of promising feature transformations simply because they perform poorly under default model configurations. To address this, we propose a collaborative tuning framework that performs an interleaved exploration of the joint space of $(T, \boldsymbol{\lambda})$ combinations to ensure the discovery of global optima.
The core challenge of joint optimization is the bidirectional dependency: BO requires a promising feature set to optimize $\boldsymbol{\lambda}$, while the LLM requires an optimized model to accurately evaluate the potential of a pipeline $T$. We resolve this through a \textbf{mutual conditioning mechanism}.

\subsubsection{BO-based HPO Conditioned on FE}
\label{sec:bo_conditioned_on_fe}

In contrast to conventional HPO that operates on a static dataset, the BO conditioned on FE must navigate a dynamic search tree containing multiple dataset states. The primary challenge lies in modeling the joint influence of the FE pipeline and configuration to determine which specific dataset state $s$ paired with which configuration $\boldsymbol{\lambda}$, yields the global optimum. To achieve this, we redefine the two core components of BO: the surrogate model and the acquisition function optimizer.

\noindent \textbf{Surrogate model}:
Since the FE pipeline lacks an explicit, continuous search space, we utilize meta-features $\phi(s)$ to characterize the dataset state after transformations. This allows the surrogate model to map discrete tree nodes into a representative feature space. We formalize the training dataset for the surrogate model $f$ as:
% \begin{equation}
% \label{eq:bo_data}
% \mathcal{D}_{\text{BO}} = \bigcup_{S_i \in \mathcal{V}_{\text{tree}}} \left\{ \left( \left[ \phi(S_i), \boldsymbol{\lambda}_{i,j} \right], v_{i,j} \right) \right\}_j
% \end{equation}
\begin{equation}
\label{eq:bo_data}
\mathcal{D}_{\text{BO}} = \left\{ \left( [\phi(s_i), \boldsymbol{\lambda}_{i,j}], v_{i,j} \right) \mid s_i \in \mathcal{V}_{\text{tree}}, \forall j \right\},
\end{equation}
where
$\mathcal{V}_{\text{tree}}$ is the set of all nodes in the MCTS tree.
% $\phi(s_i) \in \mathbb{R}^d$ is the meta-feature vector representing the dataset state at node $s_i$.
$\boldsymbol{\lambda}_{i,j}$ and $v_{i,j}$ denote the $j$-th ML configuration and its score evaluated on node $s_i$.
We employ Random Forest (RF)~\cite{smac_hutter2011sequential} as the surrogate model due to its support for categorical variables and computational efficiency.

\noindent \textbf{Acquisition function optimizer}:
To identify the most promising $(s, \boldsymbol{\lambda})$ pair, the optimizer explores the joint space through a hybrid sampling strategy to construct a candidate pool $\mathcal{P}_\text{cand}$:
\begin{itemize}[leftmargin=2em, topsep=3pt, partopsep=5pt, itemsep=5pt, parsep=0pt]
\item \textbf{Local search}: For each node in the tree, we perform perturbations around all its historically evaluated configurations within the hyperparameter space $\Lambda$. These perturbed configurations are paired with the corresponding node's meta-features to capture local improvements.
\item \textbf{Random search}: We perform global random sampling across the hyperparameter space $\Lambda$. Each sampled configuration is combined with the meta-features of every node in the current tree to ensure a broad exploration of the joint space.
\end{itemize}
Finally, the surrogate model predicts the performance of the combinations in the candidate pool $\mathcal{P}_\text{cand}$, after which an acquisition function (e.g., Expected Improvement~\cite{ei_hernandez2014predictive}) quantifies their utility and selects the optimal combination $(s^*, \boldsymbol{\lambda}^*)$. This mechanism simultaneously recommends the next ML configuration and identifies the specific dataset state $s^*$ best positioned to leverage its potential.

\subsubsection{LLM-based FE Conditioned on HPO}
On the other hand, the MCTS-based FE search must be informed by the HPO process. We modify the original FE search logic to incorporate two levels of HPO-driven conditioning:
(i) During the Selection phase, the maximum performance $\tilde{v}^{\max}_{s'}$ in Eq.~\eqref{eq:puct_qvalue} is determined by HPO; specifically, when the BO discovers a superior validation score $v^*$ at any given node, a localized update propagates this new performance "ceiling" back through the node’s ancestors to prioritize that branch in future selection cycles.
This enables the MCTS to identify FE pipelines that exhibit the greatest synergy with optimized ML configurations.
(ii) During the Playout phase, the new child node $s_{\text{new}}$ inherits the best-performing ML configuration from its parent for evaluation, ensuring that new feature operations are immediately evaluated at their highest possible capacity rather than under default settings.

\subsection{Dynamic Optimizer Selector}
While FE and HPO exhibit strong synergy, their relative contribution to performance gains varies significantly across datasets and search stages. To optimize resource allocation, we model the selection between FE and HPO as a Multi-Armed Bandit (MAB) problem. We employ a Predictor Upper Confidence Bound (PUCB)~\cite{puct_rosin2011multi} policy to dynamically decide which optimizer to execute at each step $m$:
\begin{equation}
\label{eq:puct_selector}
a^* = \arg\max_{a \in \{ \text{FE, HPO} \}} \left( Q(a) + C_2 \cdot \omega_a(m) \frac{\sqrt{\sum_{a' \in \{ \text{FE, HPO} \}} N_{a'}}}{1 + N_a} \right),
\end{equation}
where
$N_a$ is the number of times action $a$ has been selected.
$C_2$ is a hyperparameter controlling exploration pressure.
$\omega_a(m)$ is the time-varying prior weight for action $a$ at iteration $m$.
$Q(a)$ is the exploitation term, formulated as the empirical success rate of action $a$ in achieving a performance breakthrough. 
Specifically, $Q(\text{FE})$ represents the proportion of trials where the generated FE operation yields a score surpassing that of its parent node. $Q(\text{HPO})$ is defined as the frequency with which HPO on a given node discovers a configuration that exceed the node's historical best performance.

\noindent \textbf{Prior weight scheduling}.
we define the prior weights $\omega_a(m)$ as linear functions of the search progress $m/M$ ($M$ is the total budget):
% \begin{itemize}[leftmargin=2em, topsep=3pt, partopsep=5pt, itemsep=5pt, parsep=0pt]
% \item $\omega_{\text{FE}}(m) = p_1 - \delta m$: Begins at $p_1$ and decays linearly to $0.5$.
% \item $\omega_{\text{HPO}}(m) = p_2 + \delta m$: Begins at $p_2$ and increases linearly to $0.5$.
% \end{itemize}
\begin{itemize}[leftmargin=2em, topsep=3pt, partopsep=5pt, itemsep=5pt, parsep=0pt]
\item $\omega_{\text{FE}}(m) = p_1 - (p_1 - 0.5) \frac{m}{M}$: Starts at $p_1$ and decays to $0.5$.
\item $\omega_{\text{HPO}}(m) = p_2 + (0.5 - p_2) \frac{m}{M}$: Starts at $p_2$ and increases to $0.5$.
\end{itemize}
Given $p_1 + p_2 = 1$, these prior weights can be unified using a step parameter $\delta = \frac{p_1 - 0.5}{M}$: $\omega_{\text{FE}}(m) = p_1 - \delta m$ and $\omega_{\text{HPO}}(m) = p_2 + \delta m$.

\begin{theorem}[Budget Equilibrium]
\label{thm:budget_equilibrium}
Consider the rule in Eq.~\eqref{eq:puct_selector} under a neutral reward signal ($Q(a) = \text{const}$). Let $M \in \mathbb{Z}^+$ be the total budget. If the initial bias satisfies $0.5 \le p_1 <\frac{M+1.5}{M+3}$, the linear scheduling of $\omega_{\text{FE}}$ and $\omega_{\text{HPO}}$, which converges to an equilibrium of $0.5$ at $m=M$, guarantees a balanced budget distribution:
$$
N_{\text{FE}}(M), N_{\text{HPO}}(M) \in \left\{ \left\lfloor \frac{M}{2} \right\rfloor, \left\lceil \frac{M}{2} \right\rceil \right\}
$$
\end{theorem}

The proofs are provided in Appendix~\ref{app:proof}.
By leveraging this PUCB-based mechanism, the framework achieves a dual advantage: 
(i) With prior weight, it implements a learnability warmup that prioritizes FE initially to ensure the data is "model-ready," shielding HPO from deceptive, noise-driven signals inherent in raw features. Unlike traditional decoupled methods that optimize FE and HPO sequentially, our approach facilitates a balanced co-optimization adapted by task-specific performance gains $Q(a)$.
(ii) Theorem~\ref{thm:budget_equilibrium} guarantees reaching a 5:5 budget equilibrium in neutral reward scenarios (i.e., when the exploitation term $Q(a)$ is not considered), making the actual allocation fundamentally task-governed.
In summary, FE and HPO engage in a dynamic competition for computational resources, where the final allocation is adaptively steered by task-specific empirical rewards to reach the global optimum.

\subsection{Algorithm Summary}
\sys iteratively performs joint optimization in four steps:: 
1) The optimizer selector applies the PUCB rule to determine whether to execute FE or HPO; 
2) The selected module is executed, where FE utilizes MCTS to extend the most promising feature pipeline $T$ given the current HPO state, 
while HPO employs BO to identify the optimal configuration $\boldsymbol{\lambda}$ conditioned on all dataset states; 
3) The proposed $(T, \boldsymbol{\lambda})$ pair is evaluated on the validation set to obtain a score;
4) This score is returned as a reward signal to the selector for future decisions. 
This process repeats until the total budget $M$ is exhausted, after which \sys returns the optimal pair $(T^*, \boldsymbol{\lambda}^*)$.

% \noindent \textbf{Extensibility}. 
% CoFEH establishes a paradigm for the joint optimization of LLM-based FE and BO-based HPO. Characterized by its modular extensibility, the architecture is agnostic to specific HPO implementations and maintains full compatibility with state-of-the-art ones proposed for specific scenarios~\cite{divbo_shen2022divbo,optdivbo_poduval2024cash}.

\section{Experiment}

\subsection{Experiment Setup}
\label{sec:exp_setup}

\begin{table*}[t!]
\centering
\caption{Comparison in standalone FE and joint FE+HPO scenarios. We report test error (\%) for classification and test MSE for regression across 3 runs (Mean ± Std, $\downarrow$). Best and second-best results are bolded and underlined. Bottom rows show the average rank (Avg. Rank) per scenario, while Imp. (\%) quantifies the relative error reduction of joint FE+HPO relative to standalone FE.}
\label{tab:main_results_no_std}
\setlength{\tabcolsep}{2pt}
\resizebox{\textwidth}{!}{
\begin{tabular}{l|cccccc|cccccc}
\toprule
 & \multicolumn{6}{c|}{Standalone FE} & \multicolumn{6}{c}{Joint FE + HPO} \\
Dataset & Mindware & OpenFE & OCTree & ELLM & LFG & \textbf{\sys} & Mindware & OpenFE & OCTree & ELLM & LFG & \best{\sys} \\
\midrule
\rowcolor{gray!10} \multicolumn{13}{c}{\textit{Classification Tasks}} \\
rl & \std{19.65}{0.81} & \std{19.48}{0.20} & \std{22.80}{0.38} & \textbf{\std{17.84}{1.80}} & \std{22.77}{0.81} & \underline{\std{17.84}{2.30}} & \std{23.20}{0.30} & \textbf{\std{19.52}{0.10}} & \std{22.98}{0.40} & \underline{\std{19.63}{1.90}} & \std{22.87}{0.84} & \std{21.35}{0.11} \\
electricity & \std{9.08}{0.13} & \std{8.47}{0.07} & \std{8.54}{0.10} & \underline{\std{7.71}{0.21}} & \std{8.84}{0.24} & \textbf{\std{7.31}{0.35}} & \std{8.28}{0.17} & \std{7.93}{0.07} & \std{7.76}{0.31} & \textbf{\std{6.79}{0.48}} & \std{7.65}{0.26} & \underline{\std{7.15}{0.14}} \\
compass & \std{23.54}{0.07} & \textbf{\std{19.94}{0.17}} & \std{25.34}{0.18} & \std{22.93}{0.72} & \std{22.70}{0.59} & \underline{\std{22.25}{0.08}} & \std{23.63}{0.23} & \underline{\std{20.79}{0.09}} & \std{23.44}{0.11} & \std{23.63}{0.75} & \std{21.76}{0.17} & \textbf{\std{19.57}{0.44}} \\
wine & \std{23.32}{0.25} & \underline{\std{21.43}{0.51}} & \std{22.72}{0.80} & \std{22.19}{0.16} & \std{23.13}{0.27} & \textbf{\std{19.45}{0.18}} & \std{22.70}{1.10} & \underline{\std{22.06}{0.55}} & \std{23.46}{0.91} & \std{23.54}{0.44} & \std{23.41}{0.48} & \textbf{\std{19.50}{0.62}} \\
house\_16H & \std{12.30}{0.04} & \std{13.03}{0.18} & \underline{\std{11.62}{0.10}} & \std{11.76}{0.27} & \textbf{\std{11.60}{0.18}} & \std{12.06}{0.05} & \textbf{\std{11.92}{0.24}} & \std{12.38}{0.23} & \std{12.18}{0.25} & \std{12.31}{0.42} & \underline{\std{11.96}{0.19}} & \std{12.27}{0.04} \\
Magic & \std{13.92}{0.26} & \textbf{\std{11.87}{0.23}} & \std{14.21}{0.15} & \std{14.13}{0.15} & \std{14.14}{0.05} & \underline{\std{12.86}{0.26}} & \underline{\std{13.12}{0.36}} & \textbf{\std{10.99}{0.22}} & \std{14.42}{0.15} & \std{14.39}{0.28} & \std{14.24}{0.14} & \std{13.53}{0.40} \\
higgs & \std{28.10}{0.13} & \std{29.67}{0.09} & \std{28.22}{0.01} & \std{28.18}{0.02} & \underline{\std{27.74}{0.07}} & \textbf{\std{27.14}{0.32}} & \underline{\std{27.22}{0.06}} & \std{27.24}{0.09} & \std{27.46}{0.09} & \std{27.45}{0.09} & \std{27.31}{0.11} & \textbf{\std{26.35}{0.02}} \\
jannis & \textbf{\std{20.33}{0.13}} & \std{22.03}{0.19} & \std{20.99}{0.13} & \std{21.05}{0.06} & \std{21.02}{0.08} & \underline{\std{20.68}{0.15}} & \textbf{\std{19.65}{0.12}} & \std{20.77}{0.23} & \std{20.62}{0.15} & \std{20.56}{0.29} & \std{20.41}{0.01} & \underline{\std{20.34}{0.14}} \\
credit & \std{25.10}{0.36} & \std{26.21}{0.01} & \std{24.31}{0.15} & \std{24.47}{0.28} & \textbf{\std{24.13}{0.29}} & \underline{\std{24.21}{0.03}} & \underline{\std{22.59}{0.31}} & \std{22.94}{0.07} & \std{22.95}{0.09} & \std{22.92}{0.07} & \std{22.90}{0.24} & \textbf{\std{22.58}{0.21}} \\
eye\_movements & \std{30.88}{1.70} & \std{36.96}{1.10} & \std{35.06}{0.68} & \textbf{\std{29.84}{1.20}} & \std{36.49}{1.40} & \underline{\std{30.31}{1.50}} & \std{36.02}{1.80} & \std{35.98}{0.70} & \std{35.57}{0.79} & \underline{\std{30.54}{1.50}} & \std{36.99}{1.40} & \textbf{\std{29.46}{0.95}} \\
kddCup09 & \std{21.07}{0.47} & \std{22.69}{0.59} & \std{20.61}{0.64} & \std{20.74}{0.47} & \textbf{\std{19.75}{0.63}} & \underline{\std{20.15}{0.30}} & \underline{\std{19.51}{0.23}} & \std{21.01}{0.62} & \std{20.10}{0.54} & \std{19.99}{0.27} & \std{19.68}{0.07} & \textbf{\std{19.35}{0.44}} \\
road-safety & \std{21.10}{0.01} & \std{20.92}{0.09} & \std{21.94}{0.28} & \std{21.42}{0.20} & \underline{\std{19.87}{0.68}} & \textbf{\std{19.56}{0.07}} & \std{20.49}{0.05} & \underline{\std{19.60}{0.23}} & \std{21.04}{0.28} & \std{20.69}{0.86} & \std{19.99}{0.41} & \textbf{\std{18.53}{0.28}} \\
bank-marketing & \std{21.68}{0.25} & \std{22.16}{0.27} & \textbf{\std{20.64}{0.23}} & \underline{\std{20.80}{0.36}} & \std{21.00}{0.11} & \std{21.38}{0.30} & \underline{\std{19.79}{0.13}} & \std{20.55}{0.01} & \std{20.16}{0.07} & \std{20.31}{0.67} & \std{20.15}{0.14} & \textbf{\std{19.72}{0.16}} \\
phoneme & \std{14.30}{0.18} & \underline{\std{13.51}{0.11}} & \std{13.88}{0.26} & \std{14.30}{0.18} & \std{14.16}{0.18} & \textbf{\std{12.37}{0.30}} & \std{14.16}{0.19} & \underline{\std{13.04}{0.05}} & \std{14.06}{0.33} & \std{14.30}{0.19} & \std{14.30}{0.20} & \textbf{\std{12.48}{0.10}} \\
covtype & \std{11.64}{0.57} & \textbf{\std{10.69}{0.01}} & \std{13.52}{0.22} & \std{12.99}{0.04} & \std{13.75}{1.40} & \underline{\std{11.12}{0.61}} & \std{10.07}{0.13} & \std{15.92}{11.00} & \std{10.06}{0.15} & \underline{\std{9.98}{0.51}} & \std{10.23}{1.40} & \textbf{\std{7.87}{0.23}} \\
california & \std{9.72}{0.10} & \std{9.86}{0.08} & \std{9.47}{0.12} & \std{9.58}{0.01} & \underline{\std{9.31}{0.08}} & \textbf{\std{9.05}{0.12}} & \underline{\std{9.27}{0.10}} & \std{9.37}{0.07} & \std{9.53}{0.05} & \std{9.74}{0.11} & \std{9.35}{0.05} & \textbf{\std{8.68}{0.26}} \\
kdd\_ipums\_la & \std{12.32}{0.74} & \std{12.91}{0.11} & \underline{\std{11.95}{0.16}} & \textbf{\std{11.63}{0.46}} & \std{12.28}{0.10} & \std{12.06}{0.24} & \std{11.18}{0.42} & \std{11.31}{0.13} & \std{11.59}{0.29} & \underline{\std{11.10}{0.08}} & \std{11.54}{0.10} & \textbf{\std{10.93}{0.17}} \\
MiniBooNE & \std{5.89}{0.07} & \std{6.31}{0.02} & \std{6.02}{0.04} & \std{5.97}{0.07} & \underline{\std{5.89}{0.08}} & \textbf{\std{5.62}{0.05}} & \underline{\std{5.65}{0.02}} & \std{5.92}{0.02} & \std{5.82}{0.02} & \std{5.79}{0.03} & \std{5.80}{0.03} & \textbf{\std{5.51}{0.04}} \\
pol & \std{1.95}{0.09} & \std{1.79}{0.05} & \textbf{\std{1.69}{0.04}} & \std{1.82}{0.03} & \std{1.83}{0.07} & \underline{\std{1.75}{0.19}} & \underline{\std{1.74}{0.07}} & \std{1.85}{0.12} & \std{1.74}{0.04} & \std{1.87}{0.10} & \std{1.94}{0.07} & \textbf{\std{1.42}{0.12}} \\
\midrule
\rowcolor{gray!10} \multicolumn{13}{c}{\textit{Regression Tasks}} \\
airfoil\_self\_noise & \std{2.73}{0.10} & \std{2.36}{0.03} & \std{2.57}{0.06} & \std{2.44}{0.10} & \underline{\std{2.23}{0.05}} & \textbf{\std{2.01}{0.17}} & \std{2.33}{0.18} & \std{2.38}{0.07} & \std{2.12}{0.10} & \std{2.20}{0.06} & \underline{\std{1.95}{0.22}} & \textbf{\std{1.50}{0.10}} \\
cpu\_small & \std{11.14}{0.25} & \textbf{\std{10.12}{0.07}} & \std{11.16}{0.07} & \std{11.15}{0.08} & \std{11.02}{0.15} & \underline{\std{10.97}{0.07}} & \std{8.26}{0.34} & \std{8.51}{0.34} & \textbf{\std{7.85}{0.46}} & \underline{\std{7.85}{0.58}} & \std{7.97}{0.45} & \std{9.83}{0.67} \\
diamonds ($\times 10^{5}$) & \std{3.18}{0.01} & \underline{\std{2.84}{0.01}} & \std{3.04}{0.04} & \std{2.98}{0.06} & \std{2.85}{0.02} & \textbf{\std{2.72}{0.04}} & \std{2.92}{0.01} & \underline{\std{2.68}{0.02}} & \std{2.90}{0.02} & \std{2.87}{0.05} & \std{2.84}{0.01} & \textbf{\std{2.60}{0.05}} \\
plasma\_retinol ($\times 10^{4}$) & \std{5.77}{0.58} & \std{5.43}{0.09} & \underline{\std{4.54}{0.40}} & \textbf{\std{4.35}{0.21}} & \std{4.64}{0.04} & \std{5.83}{0.37} & \std{5.99}{0.33} & \std{4.65}{0.30} & \std{4.82}{0.04} & \std{4.76}{0.29} & \textbf{\std{4.07}{0.27}} & \underline{\std{4.37}{0.30}} \\
forest-fires ($\times 10^{3}$) & \std{4.43}{0.62} & \std{4.70}{0.31} & \std{4.27}{0.11} & \std{4.29}{0.19} & \textbf{\std{3.96}{0.08}} & \underline{\std{4.26}{0.02}} & \underline{\std{3.98}{0.04}} & \std{4.00}{0.00} & \std{4.16}{0.15} & \std{4.28}{0.39} & \textbf{\std{3.97}{0.02}} & \std{3.98}{0.02} \\
housing ($\times 10^{9}$) & \std{2.23}{0.00} & \std{2.16}{0.00} & \std{2.16}{0.03} & \std{2.09}{0.00} & \underline{\std{2.00}{0.07}} & \textbf{\std{1.97}{0.06}} & \std{2.06}{0.01} & \underline{\std{1.99}{0.02}} & \std{2.10}{0.04} & \std{2.04}{0.03} & \std{1.99}{0.06} & \textbf{\std{1.89}{0.01}} \\
bike ($\times 10^{3}$) & \std{1.65}{0.01} & \std{1.67}{0.00} & \std{1.67}{0.02} & \std{1.66}{0.01} & \underline{\std{1.62}{0.02}} & \textbf{\std{1.52}{0.02}} & \std{1.56}{0.01} & \std{1.60}{0.01} & \std{1.56}{0.02} & \underline{\std{1.53}{0.01}} & \std{1.54}{0.01} & \textbf{\std{1.43}{0.03}} \\
crab & \std{5.12}{0.10} & \std{5.22}{0.05} & \underline{\std{5.11}{0.15}} & \std{5.24}{0.18} & \std{5.16}{0.04} & \textbf{\std{5.07}{0.07}} & \underline{\std{4.47}{0.04}} & \textbf{\std{4.45}{0.02}} & \std{4.88}{0.02} & \std{4.88}{0.09} & \std{4.58}{0.04} & \std{4.47}{0.02} \\
insurance ($\times 10^{7}$) & \std{2.82}{0.00} & \std{2.66}{0.00} & \std{2.72}{0.18} & \std{2.52}{0.10} & \underline{\std{2.51}{0.05}} & \textbf{\std{2.45}{0.07}} & \std{1.98}{0.05} & \std{1.97}{0.04} & \std{2.01}{0.08} & \textbf{\std{1.89}{0.02}} & \underline{\std{1.90}{0.01}} & \std{1.90}{0.07} \\
\midrule
\textbf{Avg. Rank} ($\downarrow$) & 4.50 & 4.07 & 3.93 & 3.50 & \underline{3.11} & \textbf{1.82} & 3.46 & 3.86 & 4.57 & 3.93 & \underline{3.39} & \textbf{1.75} \\
\textbf{Avg. Imp.} ($\uparrow$) & / & / & / & / & / & / & \underline{+6.27\%} & +4.98\% & +4.96\% & +3.84\% & +5.13\% & \textbf{+7.03\%} \\
\bottomrule
\end{tabular}
}
\end{table*}

\noindent \textbf{Baselines}.
We evaluate CoFEH against five representative baselines, spanning both traditional and LLM-based methodologies:
--- \textit{Two traditional automated FE methods}: 
(i) \textbf{OpenFE}~\cite{openfe_zhang2023openfe} utilizes feature boosting and two-stage pruning for automated feature generation;
(ii) \textbf{Mindware}~\cite{rb_li2020efficient}, an end-to-end AutoML system that automates the entire pipeline from preprocessing to HPO with BO;
--- \textit{Three LLM-based FE methods}:
(iii) \textbf{OCTree}~\cite{octree_nam2024optimized} leverages LLMs and decision tree reasoning as linguistic feedback for iterative feature generation;
(iv) \textbf{ELLM-FT}~\cite{ellm_gong2025evolutionary} integrates evolutionary strategies with LLMs to discover optimal feature transformation sequences;
and (v) \textbf{LFG}~\cite{lfg_10.24963/ijcai.2025/782} conducts a guided tree search over a predefined operation set to generate new features via LLMs.

\noindent \textbf{Datasets and Metrics}.
Following OCTree~\cite{octree_nam2024optimized}, we incorporate 19 classification datasets benchmarked by Grinsztajn et al.~\cite{grinsztajn2022tree}. 
This is supplemented by 9 regression datasets from OpenML~\cite{bischl2025openml} and Kaggle. 
Dataset details are provided in Appendix~\ref{app:datasets}.
We use the classification error ($1-accuracy$) for classification and $mean\ squared\ error$ (MSE) for regression tasks as performance metrics.

\noindent \textbf{Downstream Models}.
We evaluate the frameworks across three downstream model scenarios:
(i) XGBoost~\cite{xgboost_chen2016xgboost},
(ii) MLP~\cite{mlp_gorishniy2021revisiting},
and
(iii) CASH~\cite{autoweka_thornton2013}: the framework must simultaneously select the best model algorithm~(e.g., RF, MLP) and its optimal parameters.
The hyperparameter spaces for all scenarios are detailed in Appendix~\ref{app:hpo_space}.

\noindent \textbf{Basic Settings}.
While it takes a different amount of time to evaluate the same model on different datasets, we use the evaluation iterations as the unit of budget.
We compare all methods under two settings:
(i) \textbf{Standalone FE}: conduct an FE search for a maximum of 100 iterations while maintaining default downstream model hyperparameters.
(ii) \textbf{Joint FE+HPO}: A total budget of 200 iterations is allocated for full-pipeline optimization. 
% Mindware \textit{searches over its unified FE and HPO space} using SMAC~\cite{smac_hutter2011sequential} (a BO variant).
Mindware \textit{leverages SMAC~\cite{smac_hutter2011sequential} (a BO variant) to simultaneously navigate the unified space of FE and HPO}.
Other decoupled baselines (OpenFE, OCTree, ELLM-FT, and LFG) \textit{execute 100 sequential rounds of FE followed by 100 rounds of HPO on the fixed optimal FE pipeline}.
Critically, to ensure a fair comparison, they employ the same SMAC-based HPO module.
In contrast, \sys performs \textit{interleaved optimization} across the entire 200-round budget.
Each dataset is split into training (60\%), validation (20\%), and test (20\%) sets. Optimization is driven by validation performance, after which the optimal FE pipeline and ML configuration are refitted on the combined training and validation set and evaluated on the test set.
We report the results across \textit{three independent runs} to ensure statistical robustness.

\noindent \textbf{Implementation Details}.
All the baselines are implemented following their official open-source versions or original methodologies. 
Within \sys, we set the exploration constants $C_1 = C_2 = \sqrt{2}$ and configure the dynamic optimizer selector  with initial prior weights $p_1 = 0.9$ and $p_2 = 0.1$, satisfying the conditions required for Theorem~\ref{thm:budget_equilibrium}. 
Our HPO component is also modified from SMAC to be conditioned on FE for collaborative optimization.
Specifically, we characterize each FE pipeline using MindWare's meta-feature set, which consolidates and extends meta-features from prior meta-
  learning studies~\cite{meta1_pfahringer2000meta,meta2_yogatama2014efficient,meta3_bardenet2013collaborative}, with details provided in
  Appendix~\ref{app:meta_features}.
All LLM-based methods use gemini-2.0-flash~\cite{gemini2flash} as the default backbone. 
Appendix~\ref{app:llm_ablation} shows that stronger backbones improve performance, but Gemini-2.0-flash offers the best cost-efficiency with only a remarkably small performance gap.
% Our codes are available in \underline{https://anonymous.4open.science/status/cofeh-159951E}.

\begin{figure*}[t!]
    \centering
    
    % 第一个子图
    \begin{subfigure}[b]{0.34\linewidth} % 宽度稍微留一点余量(0.6 -> 0.58)，防止换行
        \centering
        \includegraphics[width=\linewidth]{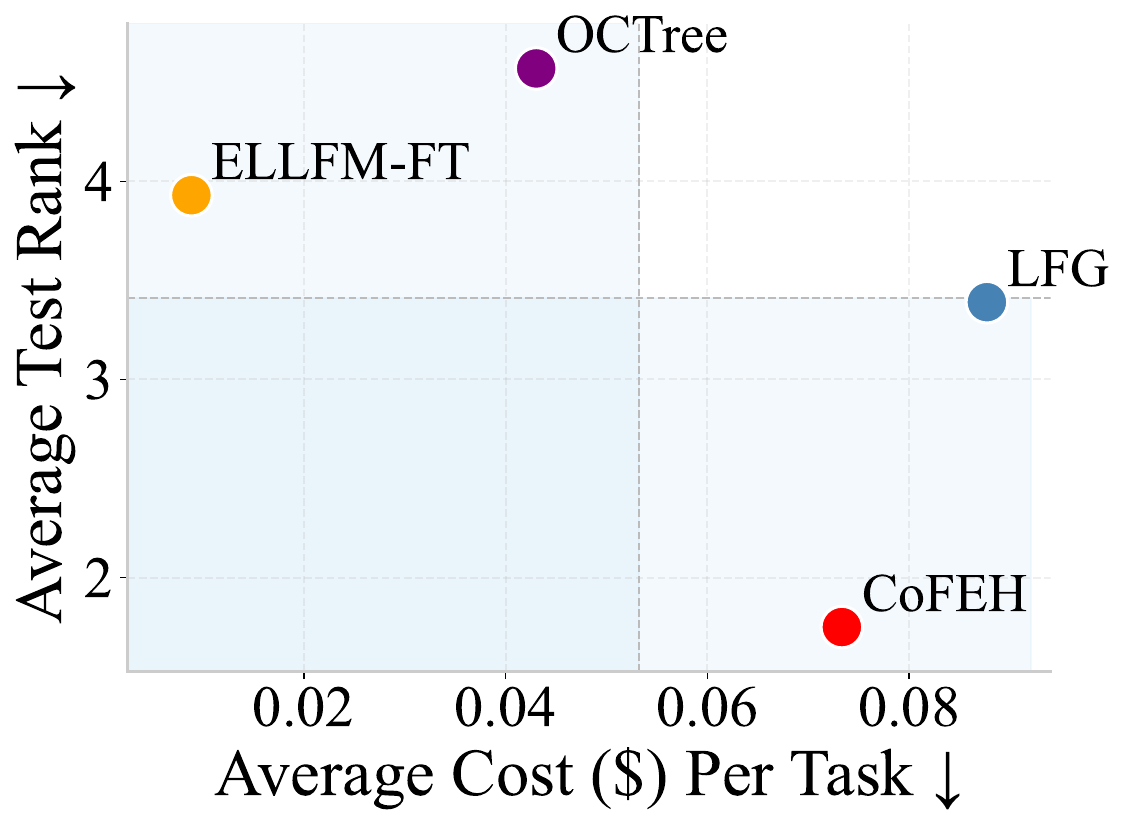}
        \caption{Expense-performance trade-off.}
        \label{fig:money_analysis}
    \end{subfigure}
    \hfill % 加上这个让两张图撑开，分布在左右两端
    % \hspace{3em}
    % 第二个子图
    \begin{subfigure}[b]{0.65\linewidth} % 宽度稍微留一点余量(0.4 -> 0.38)
        \centering
        \includegraphics[width=\linewidth]{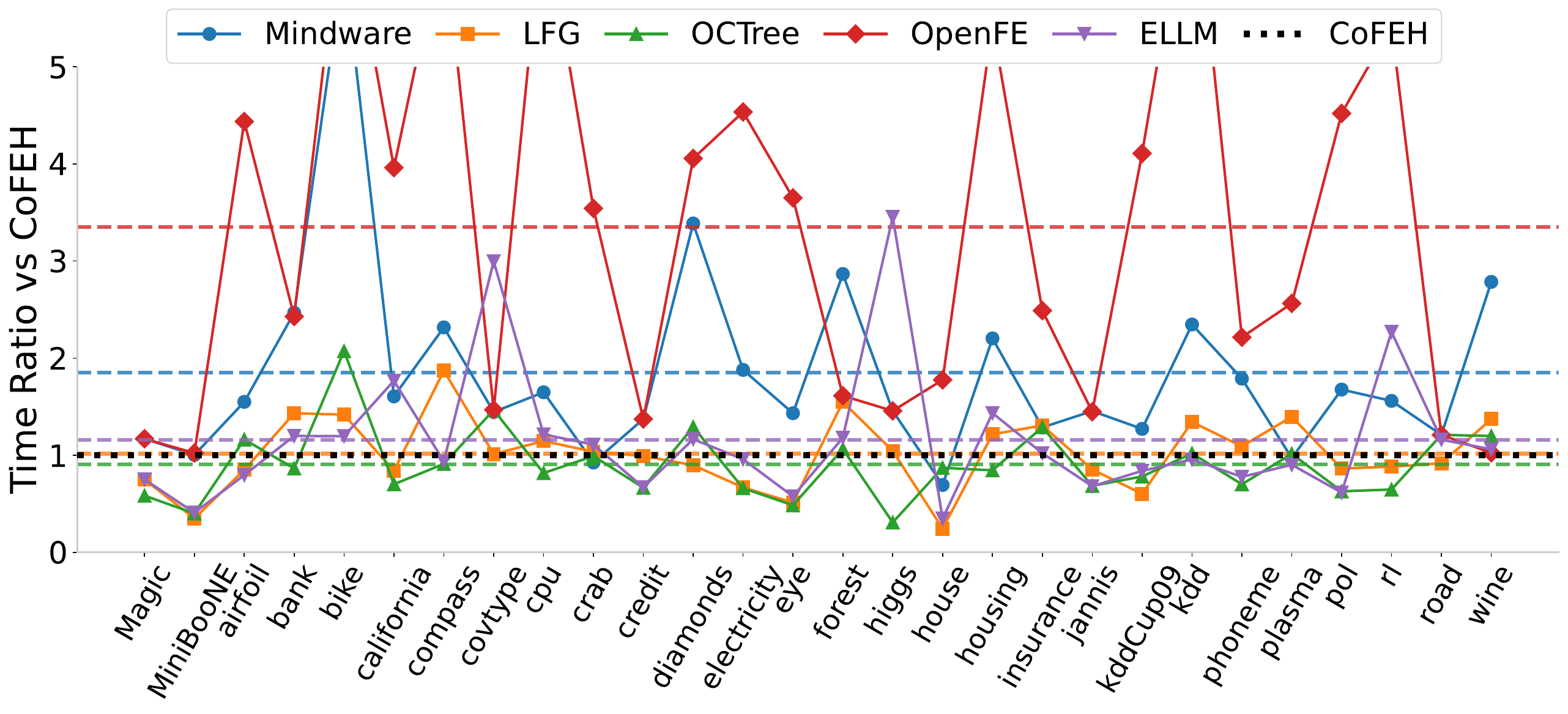}
        \caption{Per‑dataset runtime comparison (time ratio vs \sys).}
        \label{fig:time_analysis}
    \end{subfigure}
    
    \vspace{-0.5em}
    \caption{Cost analysis of the main experiment.}
    \label{fig:cost_analysis} % 记得给主图加个 label
\end{figure*}

\subsection{Main Results}
\label{sec:exp_main_results}
We first evaluate \sys and five baselines using XGBoost as the downstream ML model. The results across both standalone FE and joint FE+HPO scenarios are presented in Table \ref{tab:main_results_no_std}. Several key observations emerge from the data:
(i) \textbf{LLM-based FE vs. traditional FE}s: In the standalone FE scenario, traditional methods generally underperform compared to LLM-based approaches, with Mindware (4.52) and OpenFE (4.07) occupying the bottom average ranks. As discussed in Section \ref{sec:background_fe}, traditional frameworks are restricted to predefined, hardcoded operation sets and fixed, linear workflows.
(ii) \textbf{Performance of LLM baselines}: Among LLM-based baselines, OCTree exhibits the weakest performance (Rank 3.93). This is primarily due to its linear reasoning structure, which lacks a backtracking mechanism. ELLM (3.54) and LFG (3.11) achieve better results by employing genetic optimization and Tree-of-Thought strategies, respectively.
(iii) \textbf{Dominance of \sys in standalone FE}: \sys achieves a state-of-the-art Average Rank of 1.84, significantly outperforming the runner-up, LFG (3.11). 
A key strength lies in its ability to balance exploration and exploitation with a memory mechanism, as further validated by the ablation in Appendix~\ref{app:ee}.
Moreover, unlike other LLM baselines that focus strictly on individual feature generation, \sys enables truly free pipelines.
% We provide a \textbf{case study} of the optimal FE pipelines discovered by each method in Appendix \ref{app:FE_pipeline_example}.
The \textit{airfoil} case study in Appendix~\ref{app:FE_pipeline_example} further illustrates the qualitative difference. 
\sys synthesizes physically meaningful features such as a Strouhal-like number $St=f\cdot c/U$, geometry terms $\sin a$, and their interaction, then composes them with Yeo-Johnson transformation, standardization, and SelectKBest. 
This yields a heterogeneous pipeline spanning transformation, generation, preprocessing, and selection, whereas LLM baselines mainly remain in feature generation and OpenFE expands the feature set through many arithmetic candidates.
(iv) \textbf{Sequential vs. joint optimization}: In the joint FE+HPO scenario, the relative performance of baselines shifts significantly. Decoupled sequential baselines (OpenFE, OCTree, ELLM, and LFG) exhibit lower improvement compared to standalone FE—ranging from 3.84\% to 5.13\%—as they freeze the FE pipeline before HPO, often trapping the system in a sub-optimal state. In contrast, Mindware achieves a notable 6.27\% improvement by performing joint tuning over a unified search space. This holistic approach allows Mindware to overcome its poor FE performance (Rank 4.52), climbing to third place (Rank 3.46) in the end-to-end scenario.
(v) \textbf{Superiority of \sys in joint tuning}: Through combining FE and HPO, \sys achieves a dominant improvement rate of 7.03\% over standalone FE, further widening its lead with an Average Rank of 1.75. This success stems from our interleaved tuning strategy and a conditioned mechanism that facilitates mutual feedback between FE and HPO.
In summary, \sys not only establishes a new benchmark for FE but also achieves a powerful ``strong-strong'' synergy in joint tuning. 
% Our results demonstrate that the collaborative bridge between feature and model spaces is the key to unlocking peak performance.
Appendix~\ref{app:dataset_analysis} highlights \sys's robust scalability as dataset scale increases.
We further provide a Friedman test with Nemenyi post-hoc analysis in Appendix \ref{app:stat_analysis} to validate the consistent superiority of \sys.

% \noindent 
% \textbf{Cost analysis.}
% Appendix~\ref{app:cost_analysis} reports API and time costs for all methods: \sys averages only \$0.07 per task, with runtime comparable to LLM-based baselines and better than traditional methods. 

\noindent\textbf{Cost and runtime.}
Fig.~\ref{fig:money_analysis} compares the API expense and average test rank of LLM-based methods. \sys achieves the best cost--performance
  trade-off, attaining the best average rank with only \$0.07 per task; in contrast, LFG obtains the second-best rank but incurs the highest API
  cost, while OCTree and ELLM are cheaper but less effective. 
Fig.~\ref{fig:time_analysis} further compares Per-datase end-to-end runtime normalized by \sys.
  LLM-based baselines have comparable runtimes, whereas traditional search-heavy methods are much slower: OpenFE and Mindware require on average
  3.35$\times$ and 1.85$\times$ the runtime of \sys, respectively. 
This efficiency gap is partly due to feature explosion in traditional automated
  search: for example, OpenFE expands the ``airfoil'' dataset to 127 dimensions, whereas \sys and other reasoning-based methods typically keep the
  representation below 20 dimensions (Appendix~\ref{app:FE_pipeline_example}). 
This highlights the practical value of LLM-based FE: semantic reasoning can avoid the combinatorial accumulation of redundant features and
  thereby reduce the cost of repeated downstream evaluations.
Averaged over the 28 datasets, evaluation dominates the end-to-end runtime, taking 1842.23s (72\%), followed by LLM queries
  at 537.3s (21\%), BO at 127s (5\%), and MCTS bookkeeping at only 51.2s (2\%).

\begin{table}[t!]
\centering
\caption{Generalization performance ($\text{Mean} \pm \text{Std}$) across CASH and MLP scenarios. Imp. (\%) denotes the relative error reduction of \sys\ over the best baseline.}
\label{tab:generalization_imp}
\setlength{\tabcolsep}{4pt}
\small
\begin{tabular}{lcccc}
\toprule
\textbf{Dataset} & \textbf{LFG} & \textbf{Mindware} & \textbf{\sys} & \textbf{Imp. ($\uparrow$)} \\
\midrule
% --- CASH Scenario ---
\rowcolor{gray!15} \multicolumn{5}{c}{\textit{CASH Scenario}} \\
pol                & \underline{\std{1.70}{0.15}} & \std{1.87}{0.15} & \textbf{\std{1.38}{0.10}} & +18.9\% \\
wine               & \std{22.13}{0.27} & \underline{\std{20.89}{0.43}} & \textbf{\std{18.86}{0.39}} & +9.7\% \\
airfoil\_self\_noise & \underline{\std{2.72}{0.14}} & \std{2.88}{0.09} & \textbf{\std{1.49}{0.11}} & +45.1\% \\
housing ($\times 10^9$) & \underline{\std{2.12}{0.08}} & \std{2.12}{0.02} & \textbf{\std{1.78}{0.03}} & +16.0\% \\
\midrule
% --- MLP Scenario ---
\rowcolor{gray!15} \multicolumn{5}{c}{\textit{MLP Scenario}} \\
pol                & \std{1.53}{0.08} & \underline{\std{1.42}{0.07}} & \textbf{\std{1.36}{0.12}} & +4.6\% \\
wine               & \underline{\std{22.94}{0.49}} & \std{22.98}{0.36} & \textbf{\std{20.63}{0.21}} & +10.1\% \\
airfoil\_self\_noise & \underline{\std{10.49}{3.28}} & \std{10.59}{2.58} & \textbf{\std{6.73}{1.34}} & +35.9\% \\
housing ($\times 10^9$) & \underline{\std{2.70}{0.03}} & \std{2.77}{0.06} & \textbf{\std{2.66}{0.05}} & +1.3\% \\
\midrule
\rowcolor{gray!5}
\textbf{Avg. Rank} ($\downarrow$) & \underline{2.25} & 2.75 & \textbf{1.00} & -- \\
\bottomrule
\end{tabular}
\end{table}

\subsection{Generalization to Downstream Models}

While the preceding experiments demonstrate the efficacy of \sys using XGBoost, a robust FE framework should ideally exhibit model-agnostic generalization. 
In this section, we evaluate whether \sys can collaboratively optimize FE and HPO within two distinct downstream model types: deep learning (MLP) and multi-algorithm search (CASH).
To mitigate API overhead, we select four representative datasets from the original benchmark: two classification tasks (pol, wine) and two regression tasks (airfoil, housing). We evaluate \sys against the two strongest baselines in joint FE+HPO scenario from Table \ref{tab:main_results_no_std}: LLM-based LFG and traditional Mindware.

The results in Table \ref{tab:generalization_imp} demonstrate the robust generalization of \sys across three dimensions: (1) \textbf{Consistent superiority}: \sys consistently achieves the top performance across all datasets and scenarios with a perfect Average Rank of 1.00, maintaining a ranking hierarchy identical to the primary XGBoost experiments. 
(2)\textbf{ Significant error reduction}: Compared to the strongest baseline in each task, \sys achieves substantial relative error reductions, most notably reaching 45.1\% in the CASH scenario and 35.9\% in the MLP scenario.
(3) \textbf{Enhanced performance ceiling}: Compared with Table \ref{tab:main_results_no_std}, the end-to-end performance in the CASH scenario generally surpasses that of XGBoost, indicating that \sys effectively leverages multi-algorithm search to identify models that are more compatible with the synthesized features.

\subsection{Effectiveness of Collaborative Tuning}

In this section, we evaluate the impact of the collaborative tuning framework, which comprises two components: the mutual conditioning mechanism for bidirectional information exchange between FE and HPO, and the dynamic optimizer selector for resource scheduling. 
To isolate their contributions, we conduct an ablation study comparing the full \sys against three variants: 
(i) w/o Cond, where HPO feedback and meta-feature conditioning are removed, forcing both modules to generate independent proposals that are heuristically paired with the counterpart's current global best for evaluation; 
(ii) w/o Selector, which replaces dynamic scheduling with a fixed, alternating execution strategy; 
and (iii) Greedy, the sequential paradigm that follows the standard practice of performing 100 iterations of FE followed by 100 iterations of HPO. 
As illustrated in Fig.~\ref{fig:aba_colla}, which plots the average normalized best-so-far validation performance across four representative datasets, the Greedy approach yields the poorest results, confirming the sub-optimality of decoupled optimization.
We find that ablating either core mechanism significantly degrades convergence.
Ultimately, \sys achieves the best Average Test Rank of 1.25, substantially outperforming the w/o Selector (2.25), w/o Cond (2.75) variants, confirming the necessity of coupled information-resource management.
We next analyze the two components separately.

\noindent\textbf{Effect of mutual conditioning.}
The w/o Cond variant removes HPO feedback and FE-state conditioning, preventing BO from distinguishing how the same hyperparameter configuration
  behaves under different transformed datasets. The surrogate-fitting analysis in Appendix~\ref{app:meta_fit} provides direct evidence: adding FE
  meta-features raises the mean Spearman correlation between BO predictions and true performance from 0.587 to 0.691 across the 28 datasets, with
  especially large gains on FE-sensitive tasks.

\begin{figure}[t!]
  \centering
  \includegraphics[width=0.8\linewidth]{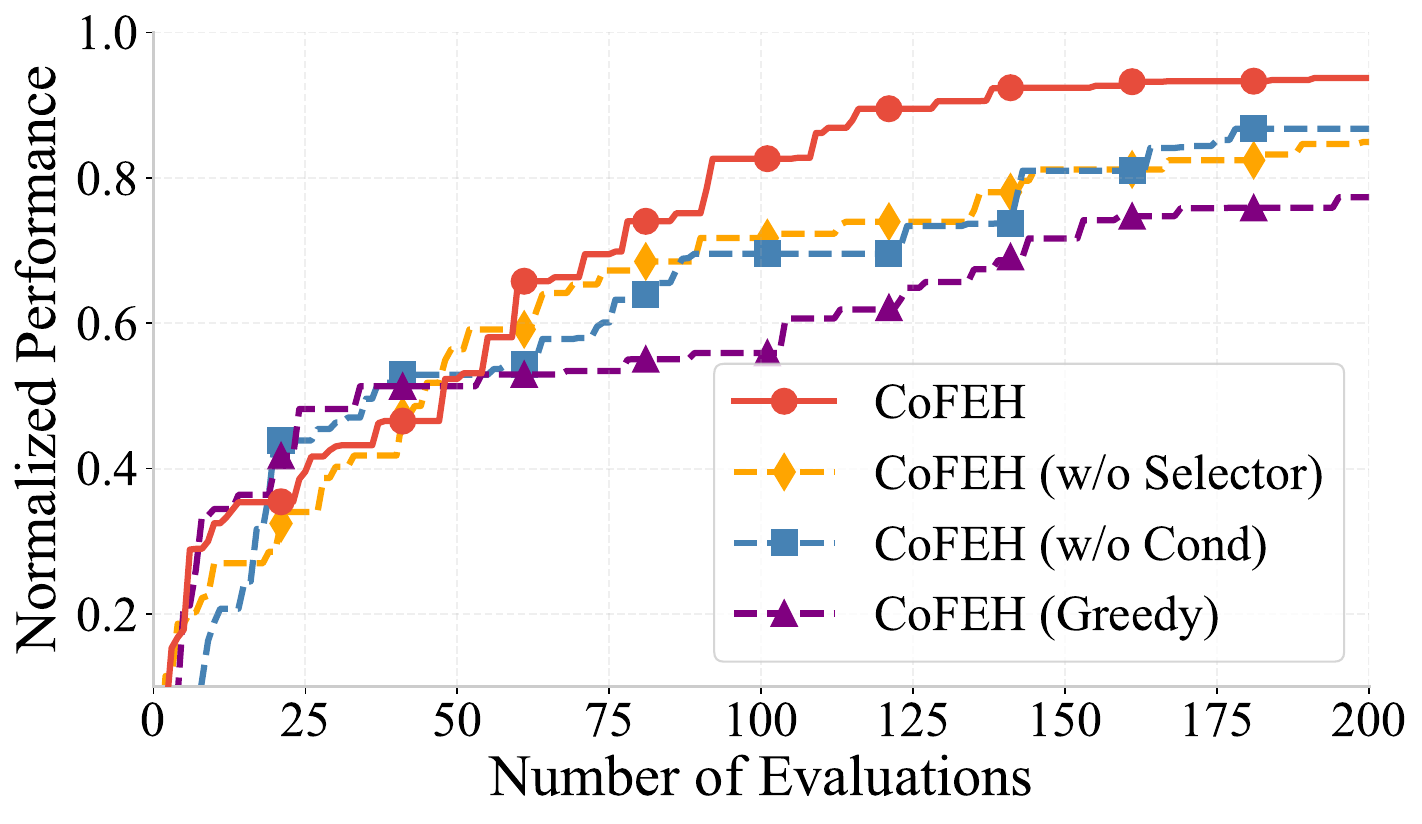}
  \caption{Ablation study of collaborative tuning.}
  \label{fig:aba_colla}
\end{figure}

\begin{figure}[t]
    \centering
    
    % 第一个子图
    \begin{subfigure}[b]{0.49\linewidth} % 宽度稍微留一点余量(0.6 -> 0.58)，防止换行
        \centering
        \includegraphics[width=\linewidth]{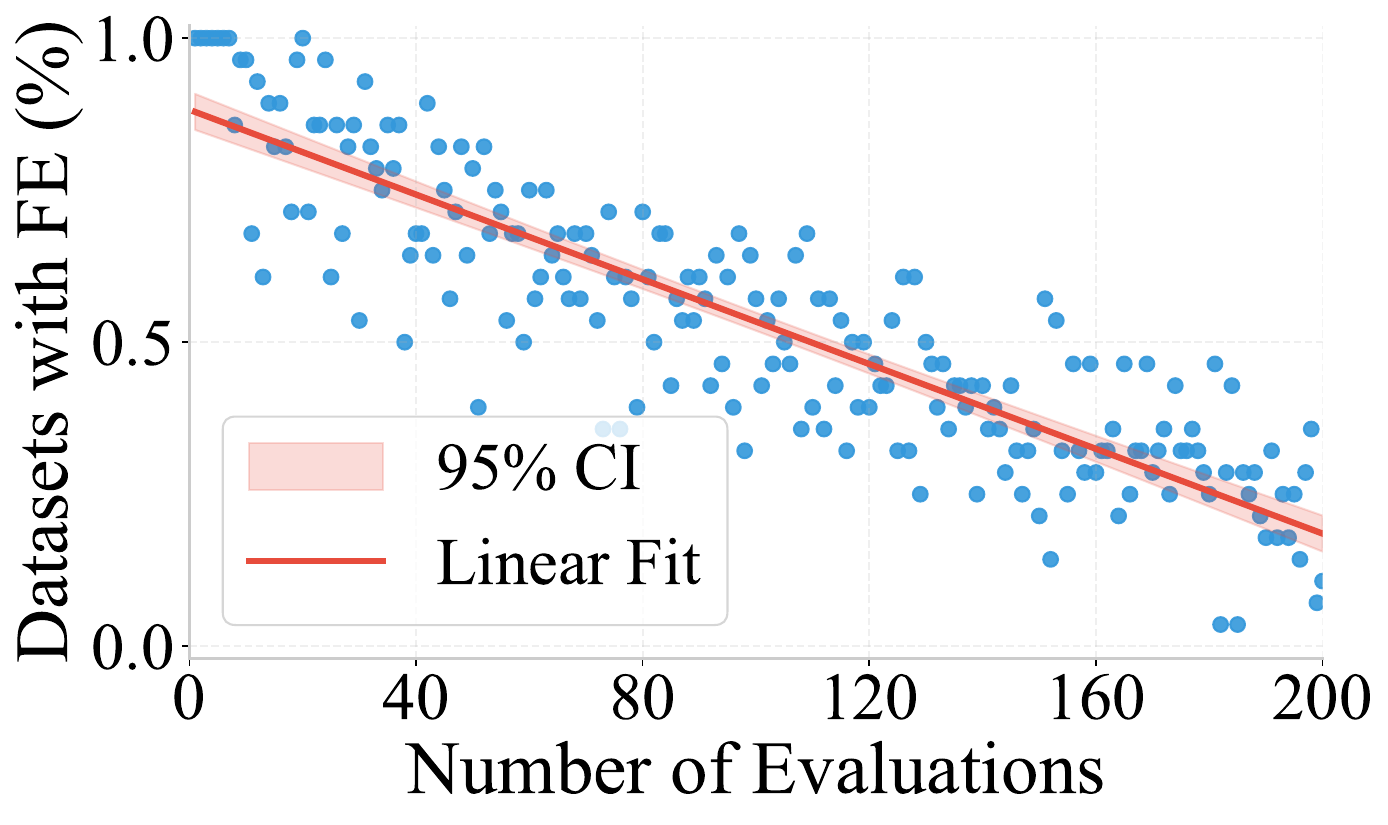}
        \caption{FE prop. across iterations.}
        \label{fig:fe_ratio_curve}
    \end{subfigure}
    \hfill % 加上这个让两张图撑开，分布在左右两端
    % \hspace{3em}
    % 第二个子图
    \begin{subfigure}[b]{0.49\linewidth} % 宽度稍微留一点余量(0.4 -> 0.38)
        \centering
        \includegraphics[width=\linewidth]{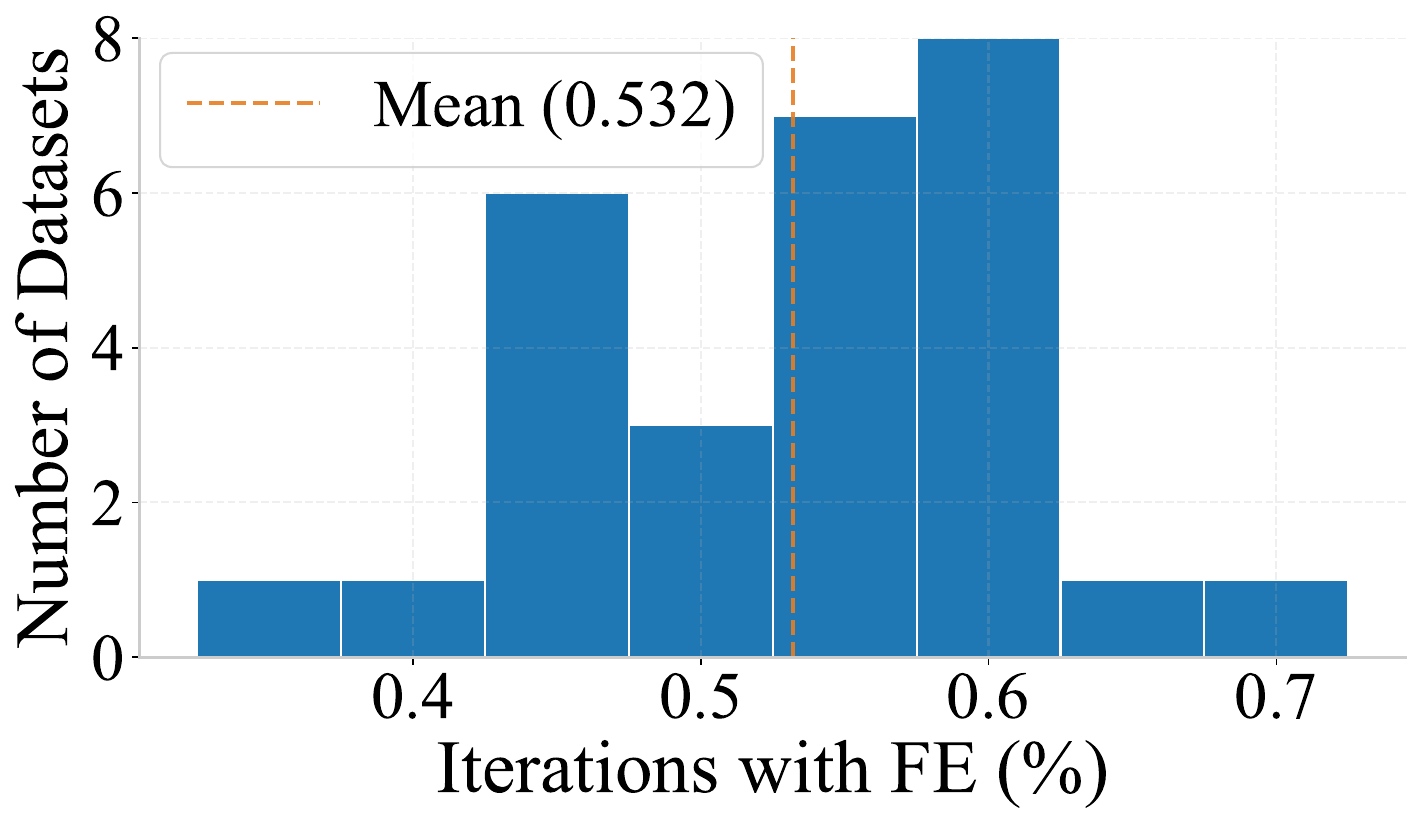}
        \caption{FE prop. across datasets.}
        \label{fig:fe_ratio_hist}
    \end{subfigure}
    
    % \vspace{-0.5em}
    \caption{FE prop. driven by the dynamic optimizer selector}
    \label{fig:budget_allocation} % 记得给主图加个 label
\end{figure}

\noindent\textbf{Effect of dynamic optimizer selection.}
To further analyze the behavior of the optimizer selector, Fig.~\ref{fig:budget_allocation} examines the distribution of FE iterations across the 28 datasets in the main experiment.
(i) \textbf{Temporal dynamics}: As shown in Fig.~\ref{fig:fe_ratio_curve}, the proportion of datasets executing FE follows a clear downward trend. This aligns with our design intuition of prior weights: prioritizing structural data refinement in early stages. Notably, both optimizers remain active throughout the process, ensuring continuous co-evolution.
(ii) \textbf{Task-adaptive allocation}: Fig.~\ref{fig:fe_ratio_hist} illustrates the distribution of total FE iterations per dataset. While the mean proportion of 0.53 empirically substantiates the budget equilibrium predicted in Theorem~\ref{thm:budget_equilibrium}, the significant variance (ranging from 0.37 to 0.7) highlights the framework's task-adaptivity.
This reflects the selector’s ability to detect whether a specific task is more sensitive to feature representations or ML configurations. 
Appendix~\ref{sec:fe_hpo_case} provides a case study of task-adaptive allocation.
We compares independent pure-FE and pure-HPO runs on two representative datasets to test whether the selector's
  allocation matches each task's dominant optimization driver. 
On ``diamonds'', pure FE achieves a higher ceiling than pure HPO, and the selector
  correspondingly allocates 70\% of the budget to FE. 
On ``jannis'', pure HPO is more effective, and the selector reduces the FE share to 37\%.
  These contrasting cases show that the selector adapts the FE/HPO budget according to the task-specific marginal utility of each optimizer.

\section{Conclusion}

In this paper, we presented \sys, a collaborative framework that bridges semantic FE discovery and numerical HPO through adaptive scheduling and mutual conditioning. 
Experiments across 28 public datasets show that \sys achieves superior performance over traditional AutoML and LLM-based baselines in both
standalone FE and joint FE+HPO settings.
The framework is also extensible: 
(i) it is agnostic to specific HPO implementations and maintains full compatibility with state-of-the-art BO variants; 
(ii) it enables seamless extension to image and text data via simple prompt modifications, offering a scalable foundation for future work.

\begin{acks}
  This work is supported by National Natural Science Foundation of China (U23B2048, U22B2037). Bin Cui is the corresponding author.
\end{acks}

\clearpage

\bibliographystyle{ACM-Reference-Format}
\bibliography{reference}

\appendix

% \section{Method Details}

\section{Prompt Design}
\label{app:prompt_design}

This appendix details our prompt engineering framework, spanning the expert-persona System Prompt, the structured User Prompt, and a concrete input/output example.

\subsubsection{System Prompt}
The System Prompt establishes the LLM's expert role and core mission:
\begin{PromptBox}[promptGray]{System Prompt}
    % \begin{promptsection}{promptGray}
    You are an elite Machine Learning Feature Engineering (FE) Expert whose mission is to analyze data distributions and synthesize high-value pipelines by generating innovative feature operations specifically designed to maximize the predictive performance of downstream ML models.
    % \end{promptsection}
\end{PromptBox}

\subsubsection{Structured User Prompt}
The User Prompt comprises the global task description and the four core components defined in Eq.~\eqref{eq:llm_reasoning}, which are detailed below:

\noindent \textbf{(i) Task description} defines the objective and output format:
\begin{PromptBox}{Template: Task Description} 
\begin{promptsection}{taskblue} 
Given the dataset metadata, transformation history, and elite memory, your task is to synthesize novel feature engineering operations to maximize the performance of the downstream model $\texttt{\{}$…$\texttt{\}}$. You should provide a reasoning insight followed by executable Python code. 

\textbf{Output Format:}
You must provide your insight followed by the code in the following structure:

--- Reason: (Explain why this specific transformation is beneficial for the downstream model.)

--- Way: (Detail the exact columns and the specific method required for the transformation.)

--- Implementation: (Provide a standalone Python function. Use the following template)
\begin{verbatim}
def your_fe_name(df):
    '''A concise description of the operator.'''
    import pandas as pd
    import numpy as np
    # Your implementation here
    return df
\end{verbatim}
\end{promptsection} 
\end{PromptBox}

\noindent\textbf{(ii) Dataset information ($\psi$)}.
The Dataset Information component provides a comprehensive data profile, grounding the LLM’s reasoning in both statistical metadata and semantic context.
To ground the LLM’s reasoning in the dataset's specific characteristics, we transform the raw data into a structured dataset information block using a multi-level extraction logic:

\begin{itemize}[leftmargin=2em, topsep=6pt, itemsep=6pt]
\item \textbf{Global profile}: Captures the dataset scale (samples/features), task type (e.g., regression/classification), the numerical-to-categorical ratio, and the semantic meaning of the dataset (if such a description is available), e.g., what the records represent and what the target variable measures.
\item \textbf{Missingness \& quality analysis}: We calculate the missing ratio for every column and flag "high-missing" features that exceed a predefined threshold (e.g., 30\%).
\item \textbf{Heuristic key feature selection}: To manage the LLM's context window, we score features based on their missingness, presence of domain descriptions, and types. We then prioritize the Top-$K$ most informative features for detailed display.
\item \textbf{Feature-Level metadata}: Distills specific statistics based on the feature type:
\begin{itemize}[leftmargin=2em, topsep=3pt, itemsep=5pt]
\item Categorical: Includes the number of unique classes and the frequency of the Top-3 categories.
\item Numerical/discrete: Includes statistical moments (mean, standard deviation) and the value range (min, max).
\end{itemize}
We append human-readable notes (e.g., feature meanings) to the metadata whenever domain knowledge is available.
\end{itemize}

\begin{PromptBox}[taskblue]{Template: Dataset Information} 
% \begin{promptsection}{taskblue} 
\textbf{[Dataset Profile]} 
$\texttt{\{}$dataset scale, …$\texttt{\}}$ \\
\textbf{Target:} $\texttt{\{}$…$\texttt{\}}$ \\
\textbf{All columns:} $\texttt{\{[}$…$\texttt{]\}}$

\textbf{[Feature Summary]} 
$\texttt{\{}$highlights data quality, specifically identifying columns with high missing-value ratios.$\texttt{\}}$

\textbf{[Key Features]} detailed metadata for each feature:

-- \texttt{col1}: type=numerical, missing=$\{\%\}$, mean=$\{\mu\}$, std=$\{\sigma\}$, range=($\{min, max\}$), note=$\{description\}$

-- \texttt{col2}: type=categorical, missing=$\{\%\}$, classes=$\texttt{\{[}$…$\texttt{]\}}$, top=$\{[(val, freq), \dots]\}$, note=$\{description\}$

-- \dots (repeated for prioritized features)

% \end{promptsection}
\end{PromptBox}

\noindent \textbf{(iii) Ancestor FE pipeline ($T_{\text{anc}}$)}
summarizes the FE pipelines from the original dataset to the currently selected node:

\begin{PromptBox}[taskblue]{Template: Ancestor FE Pipeline}
% \begin{promptsection}{taskblue}
\textbf{Current Feature Engineering pipelines:} \\
$\rightarrow$ \\
\textbf{[Do nothing]} \\
\hspace*{1em} $\llcorner$ Score: $\{Score_0\}$ \\
$\rightarrow$ \\
\textbf{[FE Operation 1]} \\
\hspace*{1em} $\vdash$ \textbf{Reason:} $\{Reason_1\}$ \\
\hspace*{1em} $\vdash$ \textbf{Way:} $\{Method_1\}$ \\
\hspace*{1em} $\llcorner$ \textbf{Score:} $\{Score_1\}$ \\
$\rightarrow$ \\
\dots \\
$\rightarrow$ \\
\textbf{[FE Operation $k$]} \\
\hspace*{1em} $\vdash$ \textbf{Reason:} $\{Reason_k\}$ \\
\hspace*{1em} $\vdash$ \textbf{Way:} $\{Method_k\}$ \\
\hspace*{1em} $\llcorner$ \textbf{Score:} $\{Score_k\}$
% \end{promptsection}
\end{PromptBox}

\noindent \textbf{(iv) Memory of good FE operations ($\mathcal{M}_{s_\text{base}}$)}
shows good feature engineering operations from history that you can reference in the next step:

\begin{PromptBox}[taskblue]{Template: Memory of Good FE Pipelines}
% \begin{promptsection}{taskblue}
\textbf{High-performing historical FE operations (memory):} \\
\textbf{[Good Operation 1]} \\
\hspace*{1em} $\vdash$ \textbf{Reason:} $\{Reason_1\}$ \\
\hspace*{1em} $\vdash$ \textbf{Way:} $\{Method_1\}$ \\
\hspace*{1em} $\llcorner$ \textbf{Score:} $\{Score_1\}$, \textbf{relative improve:} $\{RelImprove_1\}$ \\
\textbf{[Good Operation 2]} \\
\hspace*{1em} $\vdash$ \textbf{Reason:} $\{Reason_2\}$ \\
\hspace*{1em} $\vdash$ \textbf{Way:} $\{Method_2\}$ \\
\hspace*{1em} $\llcorner$ \textbf{Score:} $\{Score_2\}$, \textbf{relative improve:} $\{RelImprove_2\}$ \\
\dots \\
\textbf{[Good Operation $m$]} \\
\hspace*{1em} $\vdash$ \textbf{Reason:} $\{Reason_m\}$ \\
\hspace*{1em} $\vdash$ \textbf{Way:} $\{Method_m\}$ \\
\hspace*{1em} $\llcorner$ \textbf{Score:} $\{Score_m\}$, \textbf{relative improve:} $\{RelImprove_m\}$
% \end{promptsection}
\end{PromptBox}

\noindent\textbf{(v) Directive and Optimization Objectives}

Depending on the directive $d$, the LLM is given optimization objectives that explicitly steer its feature engineering strategy. The directive governs whether the LLM should focus on generating an initial strong operation, exploring novel transformations, or exploiting previously successful experience.

\begin{PromptBox}[taskblue]{Template: Directive and Objectives}
    % --- 通用目标 ---
    % \begin{promptsection}{taskblue}
        % \textbf{Directive:} \{d\} $\in$ \{\textsc{Initialization}, \textsc{Exploration}, \textsc{Exploitation}\}

        \textbf{Your general objectives (for any directive $d$):}\\
        1. Review and summarize the existing feature engineering operations from the original dataset to the currently selected node in the ancestor FE pipeline. \\
        2. Avoid duplicating the feature engineering approaches that have already been attempted in the ancestor FE pipeline. \\
        3. Propose one new feature engineering step. You may reference (but are not limited to) the following categories: \\
        \hspace*{1em} -- Generator: creating new features. \\
        \hspace*{1em} -- Selector: choosing or filtering the most relevant features. \\
        \hspace*{1em} -- Transformation: modifying distributions or encoding categories. \\
        \hspace*{1em} -- Rescaler: normalizing or standardizing features. \\
        \hspace*{1em} -- Imputer: filling missing data via statistical or model-based methods.
    % \end{promptsection}

    \vspace{8pt}

    4. Your strategy is \textbf{\texttt{\{d\}}} (further specifies the optimization mode and how the above objectives should be prioritized):

    \vspace{8pt}

    % --- Initialization ---
    \begin{strategysection}{warmuporange}{warmupborder}
        \textbf{Directive: \textsc{Initialization}} (for root node $s_0$) \\
        \textbf{Instruction}: Propose a \textbf{high-quality initial} FE operation for the original dataset. Focus on robust, broadly useful transformations rather than highly specialized or overly complex operations. Aim to establish a strong baseline that later operations can further improve upon.
    \end{strategysection}

    \vspace{6pt}

    % --- Exploration ---
    \begin{strategysection}{exploreyellow}{exploreborder}
        \textbf{Directive: \textsc{Exploration}} \\
        \textbf{Instruction}: Propose FE operations that \textbf{explores new regions} of the transformation space. Prioritize novel or less-explored ideas that are \textbf{distinct from existing and previously attempted methods}. Encourage diversity and bold changes to uncover potentially high-performing operations.
    \end{strategysection}

    \vspace{6pt}

    % --- Exploitation ---
    \begin{strategysection}{exploityellow}{exploitborder}
        \textbf{Directive: \textsc{Exploitation}} \\
        \textbf{Instruction}: Propose FE operations that \textbf{refines and exploits} previously successful operations. Try to \textbf{reuse, adapt, or combine} high-performing historical FE operations from the memory to further improve performance. Focus on \textbf{incremental refinement} rather than broad exploration.
    \end{strategysection}
\end{PromptBox}

\section{Proof of Theorem}
\label{app:proof}

In this appendix, we provide a detailed proof of Theorem~\ref{thm:budget_equilibrium}.
\begin{proof}

To establish the budget equilibrium property, we first formalize the conditions provided in the theorem. We assume a neutral reward signal $Q(\text{FE}) = Q(\text{HPO}) = \text{const}$ and an exploration constant $C_2 > 0$. The total iteration budget $M$ is an integer, and the prior weights are governed by the following linear schedules:
\begin{equation*}
\omega_{\text{FE}}(m) = p_1 - \delta m, \quad \omega_{\text{HPO}}(m) = p_2 + \delta m, \quad  \delta = \frac{p_1 - 0.5}{M},
\end{equation*}
where $p_1$ is subject to the boundedness constraint $0.5 \le p_1 < \frac{M+1.5}{M+3}$.

\noindent{\large \textbf{1. Simplification of the Decision Rule}}

Under the neutral reward assumption, the PUCT-based selection rule in Eq.~\eqref{eq:puct_selector} reduces to maximizing the prior-weighted exploration term. Specifically, the framework selects the action $a$ that maximizes $\omega_a(m) \frac{\sqrt{\sum N_{a'}}}{1 + N_a}$. Since the term $\sqrt{\sum N_{a'}}$ is identical for both arms at any given iteration $m$, the selection indicator $I(m+1)$ for FE is defined as:
$$
I(m+1) = 
\begin{cases} 
1, & \text{if } \frac{\omega_{\text{FE}}(m)}{1+N_{\text{FE}}(m)} > \frac{\omega_{\text{HPO}}(m)}{1+N_{\text{HPO}}(m)} \\
0, & \text{otherwise}
\end{cases}
,$$
where $N_{\text{FE}}(m)$ and $N_{\text{HPO}}(m)$ are the cumulative counts of FE and HPO selections, respectively, such that $N_{\text{FE}}(m) + N_{\text{HPO}}(m) = m$.
We define an auxiliary state function $Q(m)$ as:
$$
Q(m) = \omega_{\text{FE}}(m)\bigl(1+N_{\text{HPO}}(m)\bigr) - \omega_{\text{HPO}}(m)\bigl(1+N_{\text{FE}}(m)\bigr).
$$
Thus the decision rule is equivalent to:
$$
I(m+1) = 
\begin{cases} 
1, & \text{if } Q(m) > 0 \\
0, & \text{otherwise}
\end{cases}
,$$

\noindent{\large \textbf{2. Derivation of the Recurrence Relation}}

First, we establish the recurrence relation for $Q(m)$.
Given the linear definitions $\omega_{\text{FE}}(m) = p_1 - \delta m$ and $\omega_{\text{HPO}}(m) = p_2 + \delta m$, it follows that:
$\omega_{\text{FE}}(m+1) = \omega_{\text{FE}}(m) - \delta$ and $\omega_{\text{HPO}}(m+1) = \omega_{\text{HPO}}(m) + \delta$. 
The counts for the next iteration are updated based on the indicator $I(m+1)$ as $N_{\text{FE}}(m+1) = N_{\text{FE}}(m) + I(m+1)$ and $N_{\text{HPO}}(m+1) = N_{\text{HPO}}(m) + (1 - I(m+1))$.
Substituting these into the definition of $Q(m+1)$, we have:
\begin{align*}
Q(m+1) &= \omega_{\text{FE}}(m+1)\bigl(1+N_{\text{HPO}}(m+1)\bigr) \\
&\quad\quad - \omega_{\text{HPO}}(m+1)\bigl(1+N_{\text{FE}}(m+1)\bigr) \\
&= (\omega_{\text{FE}}(m) - \delta)\bigl(1 + N_{\text{HPO}}(m) + (1-I(m+1))\bigr) \\
&\quad\quad - (\omega_{\text{HPO}}(m) + \delta)\bigl(1 + N_{\text{FE}}(m) + I(m+1)\bigr)
\end{align*}

Expanding the terms and substituting $N_{\text{FE}}(m) + N_{\text{HPO}}(m) = m$ and $\omega_{\text{FE}} + \omega_{\text{HPO}} = 1$:
\begin{align*}
Q(m+1) &= \omega_{\text{FE}}(m)(1+N_{\text{HPO}}) - \omega_{\text{HPO}}(m)(1+N_{\text{FE}}) \\
&\quad\quad + \omega_{\text{FE}}(m)(1-I(m+1)) - \omega_{\text{HPO}}(m)I(m+1) \\
&\quad\quad - \delta(1+N_{\text{HPO}}+1-I+1+N_{\text{FE}}+I) \\
&= Q(m) + \omega_{\text{FE}}(m) - I(m+1) - \delta(m+3)
\end{align*}

\noindent{\large \textbf{3. Boundedness Lemma ($|Q(m)| < 1$)}}

We prove that for all $m \in \{0, 1, \dots, M\}$, the state function remains bounded within the open interval $|Q(m)| < 1$ via induction.

\noindent \textbf{Base case:} 

At $m = 0$, the cumulative counts are $N_{\text{FE}}(0) = N_{\text{HPO}}(0) = 0$. The state function evaluates to:
$$Q(0) = \omega_{\text{FE}}(0)(1) - \omega_{\text{HPO}}(0)(1) = p_1 - p_2 = 2p_1 - 1.$$
Given the condition $0.5 \le p_1 < \frac{M+1.5}{M+3}$, and knowing that $\frac{M+1.5}{M+3} < 1$ for any $M > 0$, it follows that $0 \le Q(0) < 1$. Thus, $|Q(0)| < 1$ holds.

\noindent \textbf{Inductive step:} 

Assume $|Q(m)| < 1$ for some iteration $m$. We examine the state transition to $Q(m+1)$ based on the decision rule:
% \begin{itemize}[leftmargin=2em, topsep=3pt, itemsep=5pt]

\textbf{Case 1: $Q(m) > 0$.} According to the decision rule, $I(m+1) = 1$ (FE is selected). Substituting it into the recurrence relation:
\begin{align*}
Q(m+1) &= Q(m) + \omega_{\text{FE}}(m) - 1 - \delta(m+3)
\end{align*}
Since $\omega_{\text{FE}}(m) <= 1$ and $\delta(m+3) >= 0$, it is clear that $Q(m+1) \le Q(m) < 1$. To prove the lower bound $Q(m+1) > -1$, we substituting $\omega_{\text{FE}}(m) = p_1 - \delta m$:
\begin{align*}
\omega_{\text{FE}}(m) - 1
&\;-\; \delta(m+3)
 = p_1 - \delta(2m + 3) - 1 \\
&> p_1 - \delta(2M + 3) - 1 \\
&= p_1 - \frac{p_1-0.5}{M}(2M+3) - 1 \\
&= \frac{M+1.5 - (M+3)p_1}{M} - 1.
\end{align*}

Given the theorem's constraint $p_1 < \frac{M+1.5}{M+3}$, we have $M+1.5 - (M+3)p_1 > 0$. Thus 
\begin{equation}
\label{eq:fuzhu}
\omega_{\text{FE}}(m) - 1 - \delta(m+3) > -1.
\end{equation}
As a result, $Q(m+1) > Q(m) - 1.$
Given the condition for Case 1 that $Q(m) > 0$, we have $Q(m+1) > -1$. Combined with the previously established upper bound $Q(m+1) < 1$, we conclude $|Q(m+1)| < 1$.

\textbf{Case 2: $Q(m) \le 0$.} According to the decision rule, $I(m+1) = 0$ (HPO is selected). Substituting $I(m+1) = 0$:
\begin{align*}
Q(m+1) = Q(m) + \omega_{\text{FE}}(m) - \delta(m+3).
\end{align*}
Since $\omega_{\text{FE}}(m) \le p_1 < 1$ and $Q(m) \le 0$, we have $Q(m+1) < 1$. For the lower bound, 
from Eq.~\eqref{eq:fuzhu} in Case 1, we have already established that:
$ \omega_{\text{FE}}(m) - 1 - \delta(m+3) > -1 \implies \omega_{\text{FE}}(m) - \delta(m+3) > 0. $
As a result, $Q(m+1) > Q(m).$
By the inductive hypothesis $Q(m) > -1$, we conclude $Q(m+1) > -1$. Thus, $|Q(m+1)| < 1$ holds for Case 2 as well.
% \end{itemize}

To sum up, $|Q(m)| < 1$ is satisfied for the entire search.

\noindent{\large \textbf{4. Terminal Equilibrium}}

At the terminal iteration $m=M$, the linear priors reach:
$$\omega_{\text{FE}}(M) = p_1 - \frac{p_1-0.5}{M}M = 0.5, \quad \omega_{\text{HPO}}(M) = p_2 + \frac{0.5-p_2}{M}M = 0.5$$
Substituting these into the auxiliary function $Q(M)$:
$$Q(M) = 0.5\bigl(1+N_{\text{HPO}}(M)\bigr) - 0.5\bigl(1+N_{\text{FE}}(M)\bigr) = 0.5\bigl(N_{\text{HPO}}(M) - N_{\text{FE}}(M)\bigr)$$

From the boundedness lemma, we have:
$$|0.5\bigl(N_{\text{HPO}}(M) - N_{\text{FE}}(M)\bigr)| < 1 \implies |N_{\text{HPO}}(M) - N_{\text{FE}}(M)| < 2$$
Since $N_{\text{FE}}(M)$ and $N_{\text{HPO}}(M)$ are integers, their difference $\Delta N = N_{\text{HPO}}(M) - N_{\text{FE}}(M)$ must also be an integer such that $\Delta N \in \{ -1, 0, 1 \}$. We now consider the parity of the total budget $M$:

\begin{itemize}[leftmargin=2em, topsep=3pt, itemsep=5pt]

\item \textbf{Case A (Even $M$):} If $M$ is even, then $\Delta N = M - 2N_{\text{FE}}(M)$ must be an even integer. The only even integer in the interval $(-2, 2)$ is $0$. Thus, $N_{\text{FE}}(M) = N_{\text{HPO}}(M) = M/2$.

\item \textbf{Case B (Odd $M$):} If $M$ is odd, then $\Delta N = M - 2N_{\text{FE}}(M)$ must be an odd integer. The only odd integers in the interval $(-2, 2)$ are $\pm 1$. Thus, $|N_{\text{FE}}(M) - N_{\text{HPO}}(M)| = 1$, implying the budget is split as $\lfloor M/2 \rfloor$ and $\lceil M/2 \rceil$.

\end{itemize}

In both cases, the adaptive selection rule minimizes the discrepancy between FE and HPO allocations, ensuring that:
$$N_{\text{FE}}(M), N_{\text{HPO}}(M) \in \{ \lfloor \frac{M}{2} \rfloor, \lceil \frac{M}{2} \rceil \}.$$
This completes the proof of the budget equilibrium.

\end{proof}

\section{Experimental Setup}

\subsection{Dataset Details}
\label{app:datasets}

This section provides the detailed specifications for the 28 benchmark datasets used in our experiments.

Table~\ref{tab:cls_datasets} lists the 19 classification datasets curated by Grinsztajn et al.~\cite{grinsztajn2022tree}, including their OpenML identifiers and basic dimensions.
\begin{table}[h]
\centering
\caption{Statistics for 19 classification datasets}
\label{tab:cls_datasets}
\small
\begin{tabular}{lccc}
\toprule
\textbf{Dataset} & \textbf{\# Samples} & \textbf{\# Features} & \textbf{OpenML ID} \\ \midrule
rl                              & 4,970   & 12 & 44160 \\
electricity                     & 38,474  & 8  & 44156 \\
compass                         & 16,644  & 17 & 44162 \\
wine                            & 2,554   & 11 & 44091 \\
house\_16H                      & 13,488  & 16 & 44123 \\
Magic                           & 13,376  & 10 & 44125 \\
Higgs                           & 940,160 & 24 & 44129 \\
jannis                          & 57,580  & 54 & 44131 \\
credit                          & 16,714  & 10 & 44089 \\
eye\_movements                  & 7,608   & 23 & 44157 \\
kddCup09                        & 5,032   & 45 & 44158 \\
road-safety                     & 111,762 & 32 & 44161 \\
bank-marketing                  & 10,578  & 7  & 44126 \\
phoneme                         & 3,172   & 5  & 44127 \\
covertype                       & 423,680 & 54 & 44159 \\
california                      & 20,634  & 8  & 44090 \\
kdd\_ipums\_la                  & 5,188   & 20 & 44124 \\
MiniBooNE                       & 72,998  & 50 & 44128 \\
pol                             & 10,082  & 26 & 44122 \\ \bottomrule
\end{tabular}
\end{table}

Table~\ref{tab:reg_datasets} details the regression datasets sourced from OpenML and Kaggle.

% 请确保在文档导言区添加：\usepackage{makecell}
\begin{table}[h]
\centering
\caption{Statistics for 9 regression datasets.}
\label{tab:reg_datasets}
\small
\setlength{\tabcolsep}{3pt} % 紧凑列间距
\begin{tabular}{lccc}
\toprule
\textbf{Dataset} & \textbf{Samples} & \textbf{Feat.} & \textbf{Source (ID / Link)} \\ \midrule
airfoil\_self\_noise & 1,503 & 6 & OpenML (44957) \\
cpu\_small           & 8,192 & 12 & OpenML (562) \\
diamonds            & 53,940 & 9 & OpenML (42225) \\
plasma\_retinol      & 315 & 13 & OpenML (511) \\
forest-fires        & 517 & 13 & OpenML (42363) \\
housing             & 20,640 & 9 & OpenML (43996) \\
bike                & 17,389 & 11 & OpenML (42712) \\

crab       & 3,893 & 8 & \makecell[c]{Kaggle \\ (crab-age-prediction)} \\
insurance  & 1,338 & 7 & \makecell[c]{Kaggle \\ (us-health-insurancedataset)} \\ \bottomrule
\end{tabular}
\end{table}
\subsection{Hyperparameter Search Spaces of Downstream Models}
\label{app:hpo_space}

This section describes the hyperparameter search spaces for the downstream models. For each task, we employ the SMAC optimizer to find the optimal set of hyperparameters within the budgets defined in Section~\ref{sec:exp_setup}.

For XGBoost, we adopt the hyperparameter search space used in Grinsztajn et al.~\cite{grinsztajn2022tree}. The detailed distributions and ranges are presented in Table~\ref{tab:xgb_hpo}.

\begin{table}[h]
\centering
\caption{XGBoost hyperparameter search space.}
\label{tab:xgb_hpo}
\small
\begin{tabular}{ll}
\toprule
\textbf{Parameter} & \textbf{Distribution / Range} \\ \midrule
Max depth & UniformInt [1, 11] \\
Num estimators & UniformInt [100, 6100, 200] \\
Min child weight & LogUniformInt [1, 1e2] \\
Subsample & Uniform [0.5, 1] \\
Learning rate & LogUniform [1e-5, 0.7] \\
Col sample by level & Uniform [0.5, 1] \\
Col sample by tree & Uniform [0.5, 1] \\
Gamma & LogUniform [1e-8, 7] \\
Lambda & LogUniform [1, 4] \\
Alpha & LogUniform [1e-8, 1e2] \\ \bottomrule
\end{tabular}
\end{table}

For MLP, we adopt the search space and architecture following Gorishniy et al.~\cite{mlp_gorishniy2021revisiting}. The model includes learning embeddings for categorical features and is trained with early stopping based on validation scores. The search space is detailed in Table~\ref{tab:mlp_hpo}.

\begin{table}[h]
\centering
\caption{MLP hyperparameter search space.}
\label{tab:mlp_hpo}
\small
\begin{tabular}{ll}
\toprule
\textbf{Parameter} & \textbf{Distribution / Range} \\ \midrule
Num layers & UniformInt [1, 8] \\
Layer size & UniformInt [16, 1024] \\
Dropout & Uniform [0, 0.5] \\
Learning rate & LogUniform [1e-5, 1e-2] \\
Category embedding size & UniformInt [64, 512] \\
Learning rate scheduler & \{True, False\} \\
Batch size & [256, 512, 1024] \\ \bottomrule
\end{tabular}
\end{table}

In the CASH scenario, the search space is expanded to include algorithm selection. The framework chooses among various learners (e.g., Random Forest, XGBoost, LightGBM, and MLP), with each learner associated with its respective hyperparameter search space.
As shown in Table~\ref{tab:search_space_for_algorithms}, the setting involves a total of 42 hyperparameters. This includes one high-level categorical parameter used for algorithm selection, while the specific learner subspaces contribute the remaining 41 hyperparameters (35 continuous and 6 categorical).
\begin{table}[h]
\caption{Search space for CASH. We distinguish categorical (cat) hyperparameters from numerical (cont) ones.}
\label{tab:search_space_for_algorithms}
\centering
\small
\begin{tabular}{lccc}
\toprule
\textbf{Type of Classifier / Regressor} & \textbf{\#$\lambda$} & \textbf{cat} & \textbf{cont} \\ \midrule
Random Forest       & 5  & 2 & 3  \\
Extra Trees         & 5  & 2 & 3  \\
Gradient Boosting   & 7  & 1 & 6  \\
MLP                 & 7  & 1 & 6  \\
LightGBM            & 7  & - & 7  \\
XGBoost             & 10  & - & 10  \\ \midrule
\textbf{Total (6 algos)} & \textbf{41} & \textbf{6} & \textbf{35} \\ \bottomrule
\end{tabular}
\end{table}
\subsection{Meta-Feature for FE Pipeline Characterization}
\label{app:meta_features}

To facilitate the collaborative optimization between Feature Engineering (FE) and Hyperparameter Optimization (HPO), \sys incorporates the meta-feature extraction framework from MindWare~\cite{rb_li2020efficient}, which consolidates and extends meta-features previously proposed in the meta-learning literature, including those from~\cite{meta1_pfahringer2000meta,meta2_yogatama2014efficient,meta3_bardenet2013collaborative}.

\noindent\textbf{Motivation and shared spirit}. 
In the original MindWare framework, meta-features are utilized for meta-learning to predict which algorithm (e.g., XGBoost, LightGBM) is likely to achieve the highest performance ceiling on a given dataset. Our approach shares a similar philosophy: we posit that the optimal hyperparameter configuration $\boldsymbol{\lambda}^*$ is inherently dependent on the current state of the dataset $\mathcal{X}'$. By extracting meta-features, we allow the HPO surrogate model to "perceive" the transformations made by the FE module, effectively conditioning the search space on the feature engineering trajectory.

\noindent\textbf{Implementation details}. We adopt the meta-feature suite provided in the MindWare repository\footnote{\url{https://github.com/PKU-DAIR/mindware/blob/master/mindware/components/meta_learning/meta_feature/meta_features.py}}. The extraction logic is categorized by task type:
\begin{itemize}[leftmargin=2em, topsep=3pt, partopsep=5pt, itemsep=5pt, parsep=0pt]
\item \textbf{Classification datasets}: 46 meta-features are extracted, covering statistical properties (e.g., skewness, kurtosis), information-theoretic measures (e.g., class entropy), etc.
\item \textbf{Regression datasets}: 33 meta-features are extracted, focusing on feature-target correlations and dataset dimensionality characteristics, etc.
\end{itemize}
Detailed definitions and the underlying source code for each metric are available in the referenced repository file.

\noindent\textbf{Orthogonality and Compatibility}. It is important to emphasize that the design of \sys is orthogonal to the specific choice of meta-features. While we utilize the MindWare suite due to its robustness and comprehensive coverage, our framework is fully compatible with any alternative meta-feature computation methodologies. This modularity ensures that as more advanced dataset characterization techniques emerge, they can be seamlessly integrated into the \sys pipeline to further enhance the collaborative optimization process.

\section{Additional Results and Analysis}

\subsection{Ablation on Different LLM Backbones}
\label{app:llm_ablation}

\begin{table}[htbp]
\centering
\caption{Test error ($\text{Mean} \pm \text{Std}$) and cost comparison averaged over 3 runs. Bold and underlined values denote the best and second-best results (lower is better). Bottom rows show the average rank and average token cost.}
\label{tab:model_comparison}
\setlength{\tabcolsep}{2pt}
\small
\begin{tabular}{lccc}
\toprule
\textbf{Dataset} & \textbf{Gemini-2.0-flash} & \textbf{Gemini-2.5-pro} & \textbf{GPT-5.2} \\
\midrule
pol                 & $\underline{\std{1.42}{0.12}}$ & \textbf{\std{1.39}{0.15}} & \std{1.42}{0.11} \\
wine                & \std{19.50}{0.62} & \underline{\std{19.42}{0.55}} & \textbf{\std{19.35}{0.68}} \\
airfoil\_self\_noise & \std{1.50}{0.10} & $\underline{\std{1.49}{0.13}}$ & \textbf{\std{1.47}{0.09}} \\
housing ($\times 10^9$) & \std{$1.89$}{0.01} & $\underline{\std{1.76}{0.01}}$ & \textbf{\std{1.72}{0.01}} \\
\midrule
\textbf{Avg. Rank}  & 2.75 & 1.75 & 1.50 \\
\textbf{Avg. Cost (\$)} & 0.088 & 1.515 & 2.123 \\
\bottomrule
\end{tabular}
\end{table}

This appendix investigates how the choice of LLM backbone affects the performance of \sys in joint FE+HPO tuning setting. We evaluate three models—GPT-5.2, Gemini-2.5-pro, and Gemini-2.0-flash—across four diverse datasets: two classification tasks (pol, wine) and two regression tasks (airfoil\_self\_noise, housing). As summarized in Table~\ref{tab:model_comparison}, all metrics are reported as averages over three independent runs to ensure statistical robustness.

Experimental results reveal a strong positive correlation between an LLM's reasoning depth and the quality of synthesized operations when exploring the FE search space conditioned on HPO. GPT-5.2 establishes the performance ceiling, achieving the best results on three out of four tasks and securing the top average rank of 1.50. This superiority suggests that advanced reasoning capabilities enable the model to propose more effective feature transformations while the HPO component simultaneously optimizes the model configuration. In contrast, Gemini-2.0-flash yields the lowest performance (Rank 2.75), highlighting the limitations of lightweight models in executing the sophisticated semantic reasoning required for optimal FE synthesis under varying HPO conditions.
Gemini-2.5-pro serves as a highly competitive mid-tier alternative, outperforming GPT-5.2 on the pol classification task and maintaining a second-best average rank of 1.75. 
This performance gradient suggests that the system's effectiveness scales directly with the model's ``intelligence ceiling.''

Cost analysis illustrates a sharp trade-off between performance and API expenditure. Gemini-2.0-flash is the most cost-effective option, with a mean task cost of $0.088$—approximately 1/24th that of GPT-5.2 ($2.123$). While it is numerically the weakest among the tested models in ranking, the performance gap remains remarkably small, which justifies its selection as our final backbone. This extreme efficiency enables a much broader search scope that can compensate for the lower per-proposal complexity through sheer volume and iteration. 
In summary, the performance of \sys is intrinsically linked to the reasoning capacity of the LLM backbone, but the ultimate selection of Gemini-2.0-flash reflects a strategic prioritization of cost-efficiency. As LLM research continues to deliver enhanced reasoning power at lower costs, the efficacy and scalability of \sys—leveraging high-efficiency backbones—are expected to increase proportionally.

\subsection{Effectiveness of Steerable Reasoning Directives}
\label{app:ee}

\begin{figure}[t!]
  \centering
  \includegraphics[width=0.9\linewidth]{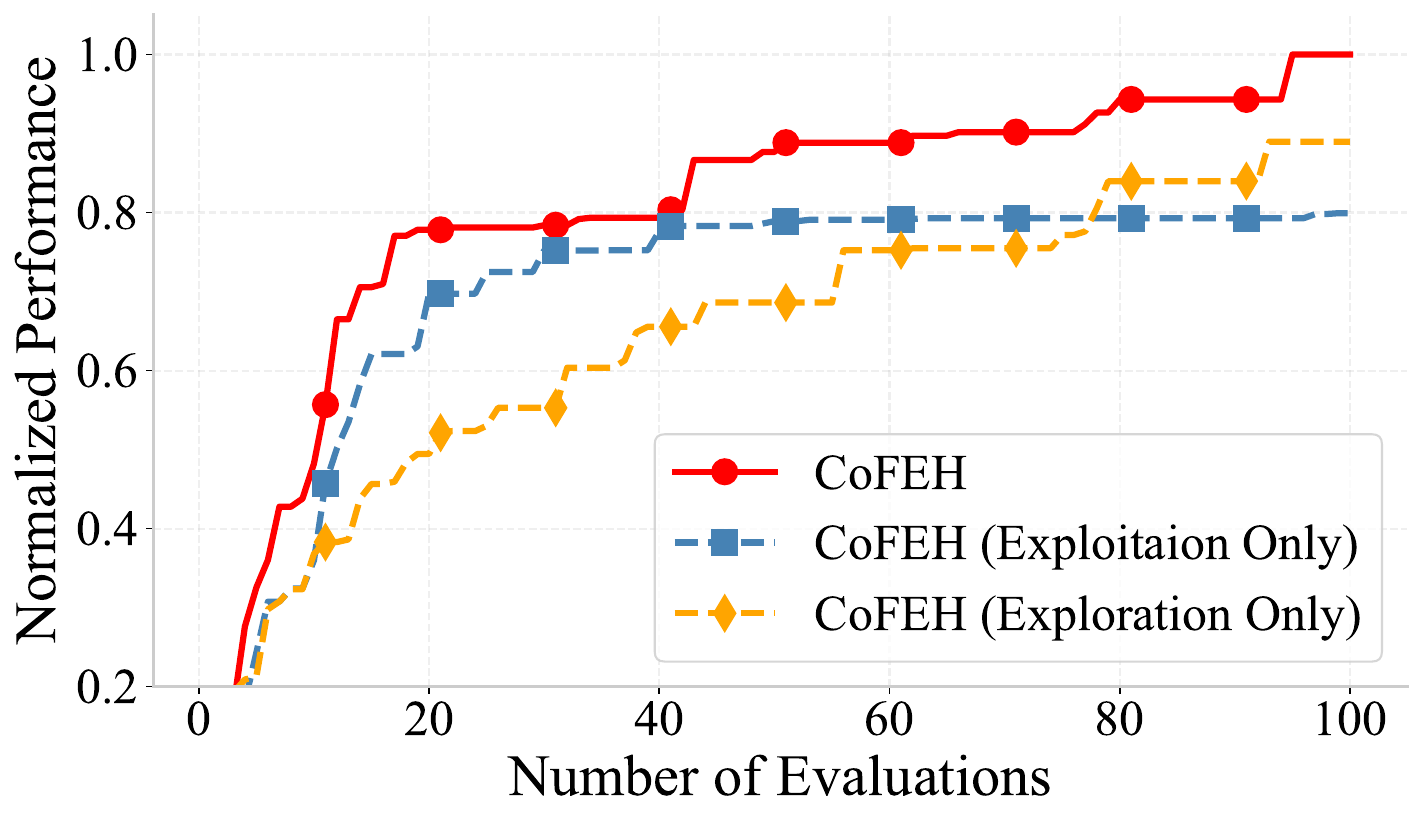}
  \caption{Abalation study of directives.}
  \label{fig:aba_directives}
\vspace{-1em}
\end{figure}

The exploration of the unbounded FE search space in \sys is explicitly driven by the interplay between the $\textsc{Exploration}$ and $\textsc{Exploitation}$ directives.
To evaluate the necessity of this dual-engine design, we conduct an ablation study by comparing the full \sys against two restricted variants:
(i) Exploration-only: The LLM performs four $\textsc{Exploration}$ expansions per node, relying solely on domain knowledge to probe the space without $\mathcal{M}_{\text{global}}$.
(ii) Exploitation-only: The LLM is restricted to four $\textsc{Exploitation}$ trials, forcing the FE process to refine ``elite experiences'' from $\mathcal{M}_{\text{global}}$ for novel transformations. To manage API expenditure while maintaining evaluation coverage across diverse task types, we select four representative datasets: two classification tasks (pol, wine) and two regression tasks (airfoil, housing). Fig.~\ref{fig:aba_directives} plots the average best-so-far validation performance across these tasks, where scores are min-max normalized for each dataset.

The results indicate that the Exploitation-only variant achieves rapid initial gains by effectively leveraging high-performing "elite experiences" from $\mathcal{M}_{\text{global}}$. However, it suffers from premature stagnation, as the search logic becomes increasingly confined to a narrow manifold of previously successful patterns, eventually trapping the process in local optima. Conversely, the Exploration-only variant exhibits significantly slower convergence; while it consistently probes novel regions using domain knowledge, the absence of iterative refinement leads to an inefficient and scattered discovery process.
\sys synthesizes the strengths of both engines to maintain a robust balance between breadth and depth throughout the search. This synergy is quantitatively substantiated by the final performance rankings: \sys achieves a dominant Average Test Rank of 1.0, significantly outperforming the Exploitation-only (2.25) and Exploration-only (2.75) variants. These findings confirm that the interplay between these two steerable directives is essential for navigating complex, unbounded search spaces efficiently under a constrained budget.
\subsection{Case Studies of Optimal FE Pipelines}
\label{app:FE_pipeline_example}

\begin{figure*}[t]
    \centering
    
    % 第一个子图
    \begin{subfigure}[b]{0.49\linewidth} % 宽度稍微留一点余量(0.6 -> 0.58)，防止换行
        \centering
        \includegraphics[width=\linewidth]{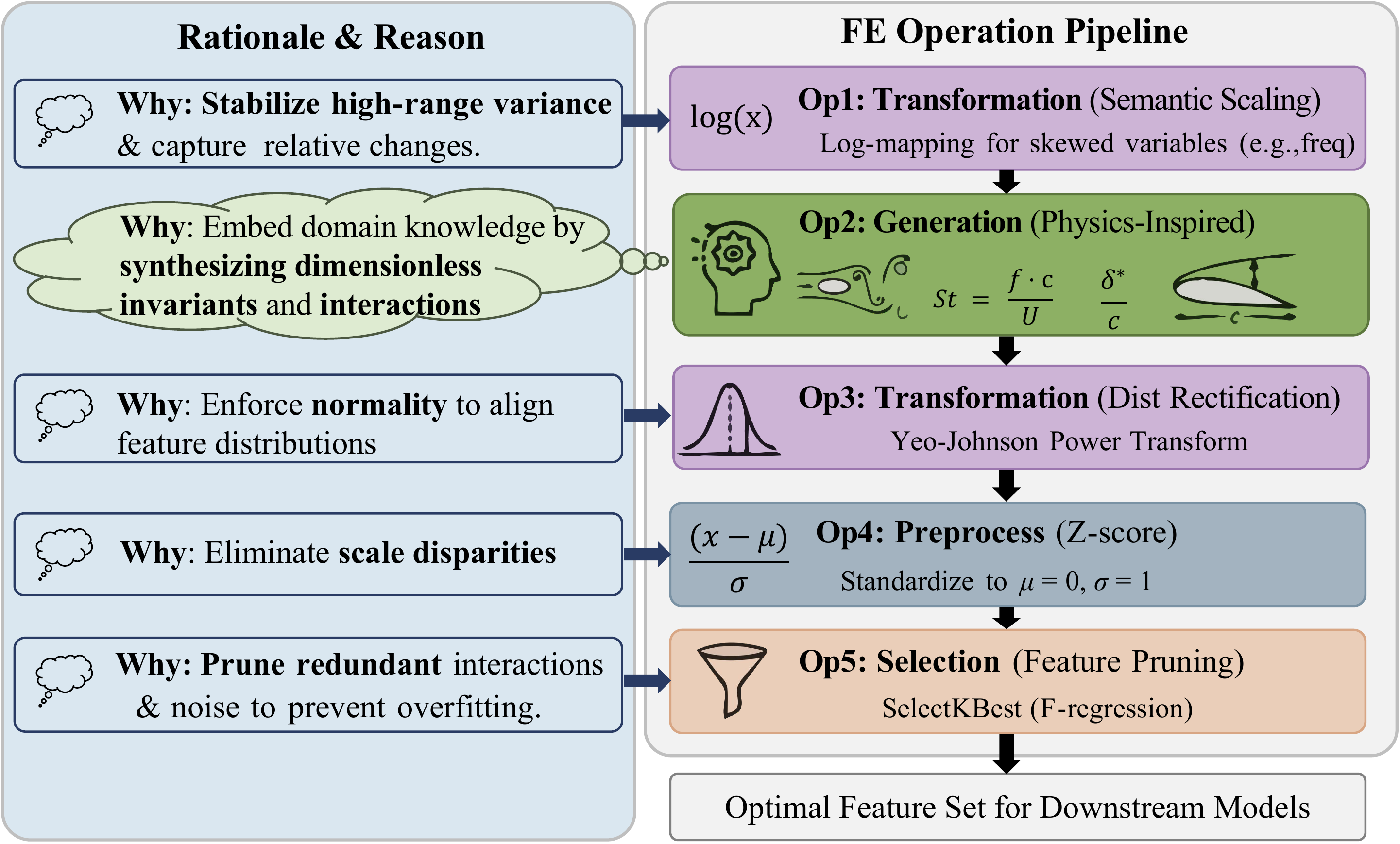}
        \caption{Optimal FE pipeline of \sys.}
        \label{fig:CoFEH_case_study}
    \end{subfigure}
    \hfill % 加上这个让两张图撑开，分布在左右两端
    % \hspace{3em}
    % 第二个子图
    \begin{subfigure}[b]{0.49\linewidth} % 宽度稍微留一点余量(0.4 -> 0.38)
        \centering
        \includegraphics[width=\linewidth]{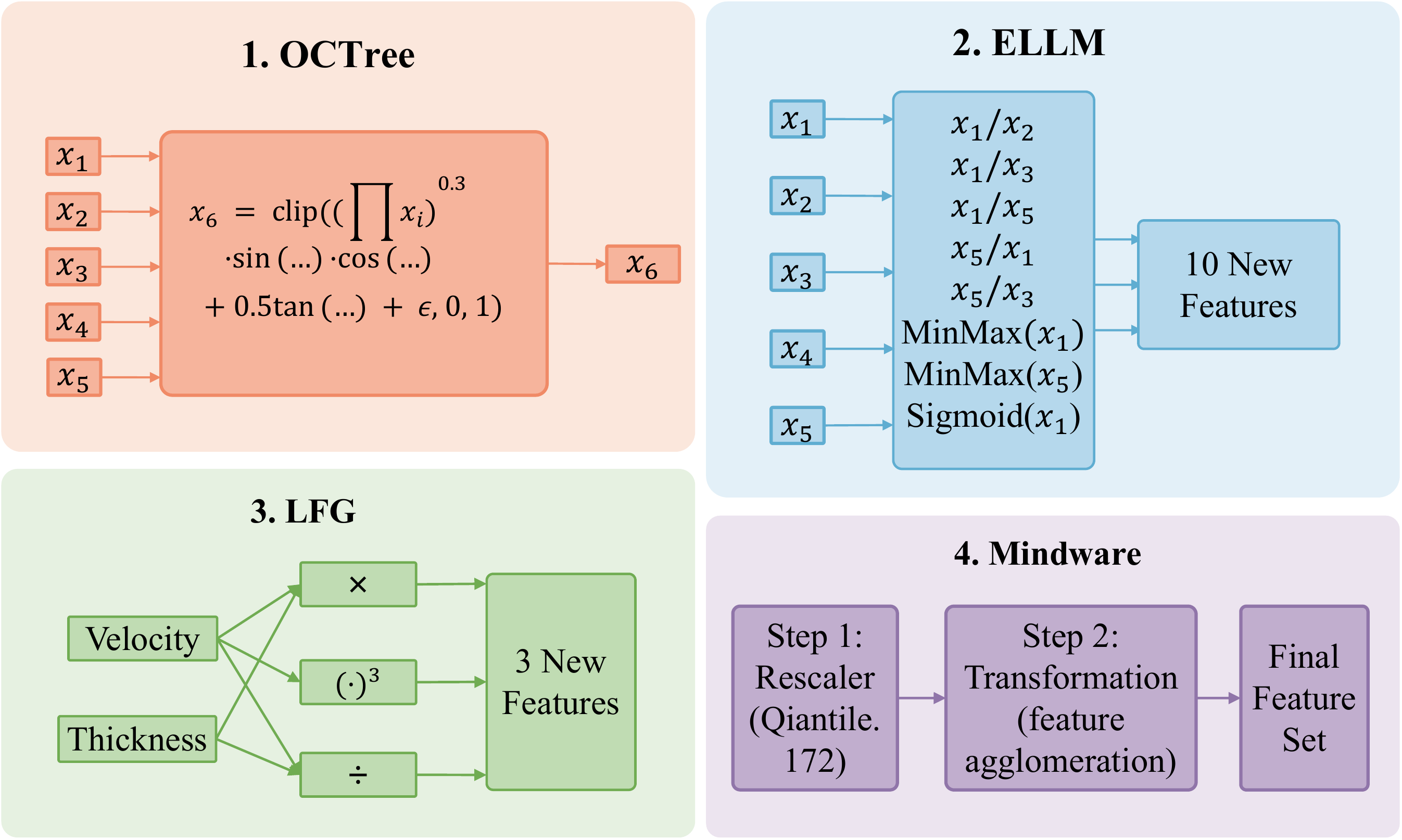}
        \caption{Optimal FE pipeline of baselines.}
        \label{fig:baseline_PP_case}
    \end{subfigure}
    
    \caption{Optimal FE pipeline discovered by all methods on ``airfoil\_self\_noise'' dataset.}
    \label{fig:case_study} % 记得给主图加个 label
% \vspace{-1em}
\end{figure*}

This appendix details the feature engineering (FE) pipelines discovered by our framework and various baselines for the NASA Airfoil Self-Noise dataset in the standalone FE scenario. The dataset, derived from anechoic wind tunnel tests of NACA 0012 airfoils, presents a challenging regression task: predicting scaled self-noise across varying chord lengths, wind speeds, and angles of attack.

\noindent \textbf{Pipeline of \sys}.
The pipeline discovered by \sys (Fig.~\ref{fig:CoFEH_case_study}) reflects a logical progression from raw data stabilization to the synthesis of physical representations. This process is structured into five distinct functional stages:

(i) \textbf{Operation 1: semantic scaling (transformation)}: The framework identifies that predictors such as frequency span several orders of magnitude. It autonomously applies a logarithmic transformation to stabilize high-range variance, effectively shifting the model’s focus toward relative spectral changes rather than absolute numerical values.

(ii) \textbf{Operation 2: physics-inspired feature generation}: Leveraging aerodynamic domain knowledge, \sys deduces that because the experimental conditions involve varying airfoil scales (chord length $c$) and flow velocities ($U$), the noise response is likely governed by dynamic similarity. It synthesizes a Strouhal-like number ($St = \frac{f \cdot c}{U}$), collapsing diverse experimental scales into a scale-invariant physical representation. The system further incorporates geometry-aware features by deriving trigonometric components ($\sin a, \cos a$) from the angle of attack. A notable emergent feature is the coupled interaction term ($St \cdot \sin a$), synthesized to model angle-modulated shedding effects—a sophisticated interaction.

(iii) \textbf{Operation 3: distribution rectification (transformation)}. Following symbolic discovery, the pipeline employs a Yeo-Johnson power transform. This step serves to rectify the distributions of the newly synthesized physical terms, enforcing Gaussian-like normality to align them statistically with the original feature space.

(iv) \textbf{Operation 4: scale standardization (Preprocess)}. The framework executes Z-score standardization ($z = \frac{x - \mu}{\sigma}$). This stage is critical for eliminating scale disparities between different features.

(v) \textbf{Operation 5: feature pruning (selection)}. The process concludes with selection via SelectKBest (F-regression). This stage distills the expanded feature bank into an optimal set, removing redundant noise to prevent overfitting and enhance model efficiency.

\sys distinguishes itself through a highly unconstrained workflow that autonomously orchestrates a fluid sequence of transformation $\rightarrow$ generation $\rightarrow$ transformation $\rightarrow$ preprocessing $\rightarrow$ selection. It also supports an essentially unbounded operation space, harmonizing foundational mathematical primitives (e.g., log-scaling) and advanced algorithmic optimizations (e.g., Yeo-Johnson transforms, F-regression) with autonomous, domain-driven synthesis.
On the other hand, Fig.~\ref{fig:baseline_PP_case} illustrates the optimal pipelines discovered by baselines, revealing a notable contrast to \sys:

\textbf{OCTree} pursues a monolithic symbolic synthesis approach, nesting multiple mathematical primitives into a single, high-order expression. While this captures complex non-linearities, the resulting single feature lacks a grounding in aerodynamic similarity laws (such as the Strouhal number), often leading to a representation that is difficult to generalize across different airfoil scales.

\textbf{ELLM and LFG} both rely on the expansion of mathematical primitives to generate multiple candidates. Crucially, the outputs of both methods are fundamentally the results of simple arithmetic operations. Furthermore, these methodologies are strictly confined to the feature generation phase, lacking a comprehensive pipeline to integrate transformation, scale standardization, or selection.

\textbf{Mindware} represents a purely two-step pipeline: a quantile-based Rescaler followed by a feature agglomeration Transformer. This result reflects a purely algorithmic paradigm that reconfigures the input space geometry through statistical motifs. Notably, it operates without any arithmetic generation, meaning it does not construct new symbolic features or physical interactions.

Finally, we consider \textbf{OpenFE}, which follows a two-stage automated feature generation. This methodology systematically explores thousands of candidate features by applying arithmetic primitives (e.g., +, -, $\times$, $\div$). to the original input space. Following this expansive generation phase, a selection process is employed, retaining 122 new features in the final dataset.

In summary, \sys distinguishes itself through its structural adaptability and operational diversity, transcending the functional constraints of all baselines through two core distinctions:
(i) \textbf{Workflow fluidity}: Unlike frameworks such as Mindware and OpenFE, which adhere to fixed, pre-defined sequences where stages are often mutually exclusive and restricted in depth, or baselines like OCTree, ELLM, and LFG, which are confined strictly to iterative feature generation regardless of pipeline length, \sys implements a truly unconstrained workflow. Although visualized as a five-stage pipeline in Fig.~\ref{fig:CoFEH_case_study}, each stage may represent a sophisticated reasoning layer where the LLM may execute multiple concurrent operations. This results in an underlying structural complexity that far exceeds what is shown.
(ii) \textbf{Synthesis of algorithmic and arithmetic operations}: \sys further bridges the gap between disparate operational paradigms. 
Mindware is purely algorithmic, relying on predefined algorithms (e.g., PCA, kernel expansions) to reconfigure the input space, while other baselines are restricted to symbolic arithmetic-operator synthesis (e.g., $+$, $-$, $\times$, $\div$).
In contrast, \sys achieves an organic integration of both worlds. It seamlessly blends domain-driven arithmetic discovery—such as the synthesis of scale-invariant physical parameters with advanced algorithmic optimizations like the Yeo-Johnson transform and F-regression. By navigating an unbounded operation space, \sys ensures that feature representations are both physically grounded and numerically optimized for downstream performance.

\begin{figure}[t!]
  \centering
  \includegraphics[width=0.8\linewidth]{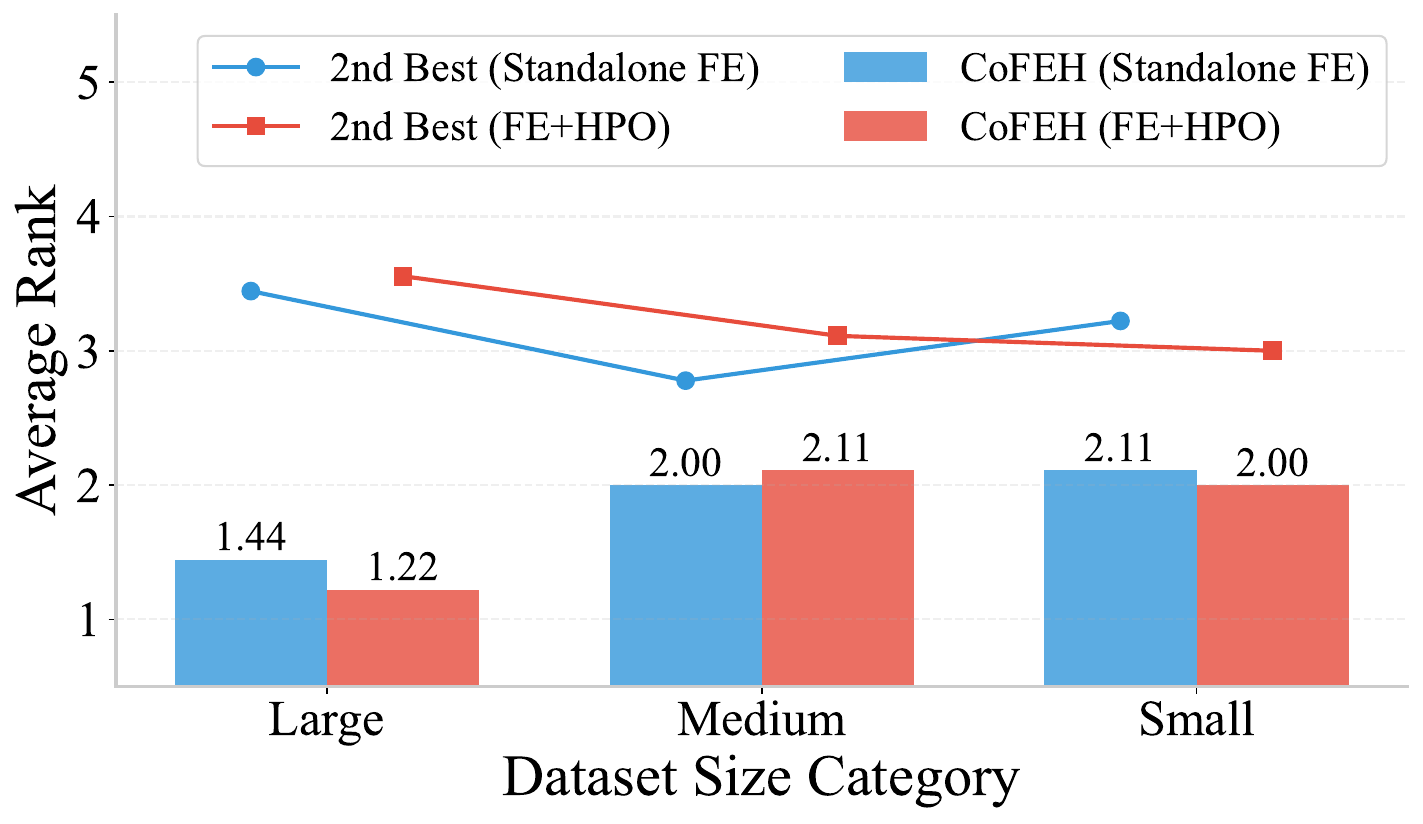}
    \vspace{-1em}
  \caption{Average rank by dataset size (samples $\times$ features).}
  \label{fig:dataset_analysis}
\end{figure}

\subsection{Impact of Dataset Characteristics}
\label{app:dataset_analysis}

In this section, we analyze the performance of \sys across different dataset scales to evaluate its robustness. We define dataset size as the product of samples and features, partitioning the 28 benchmark datasets equally into three groups: Small, Medium, and Large.
Fig.~\ref{fig:dataset_analysis}, which compares the average rank of \sys against the runner-up baseline,
(i) \sys consistently achieves the superior average rank across all categories in both standalone FE and joint FE+HPO scenarios; 
and (ii) the performance advantage of \sys becomes significantly more pronounced as the dataset scale increases, substantiating the robust scalability of our reasoning-driven approach, which effectively navigates the complex search spaces of large datasets where other methods often struggle.

\begin{figure*}[tb]
    \centering
    
    % 第一个子图
    \begin{subfigure}[b]{0.4\linewidth} % 宽度稍微留一点余量(0.6 -> 0.58)，防止换行
        \centering
        \includegraphics[width=\linewidth]{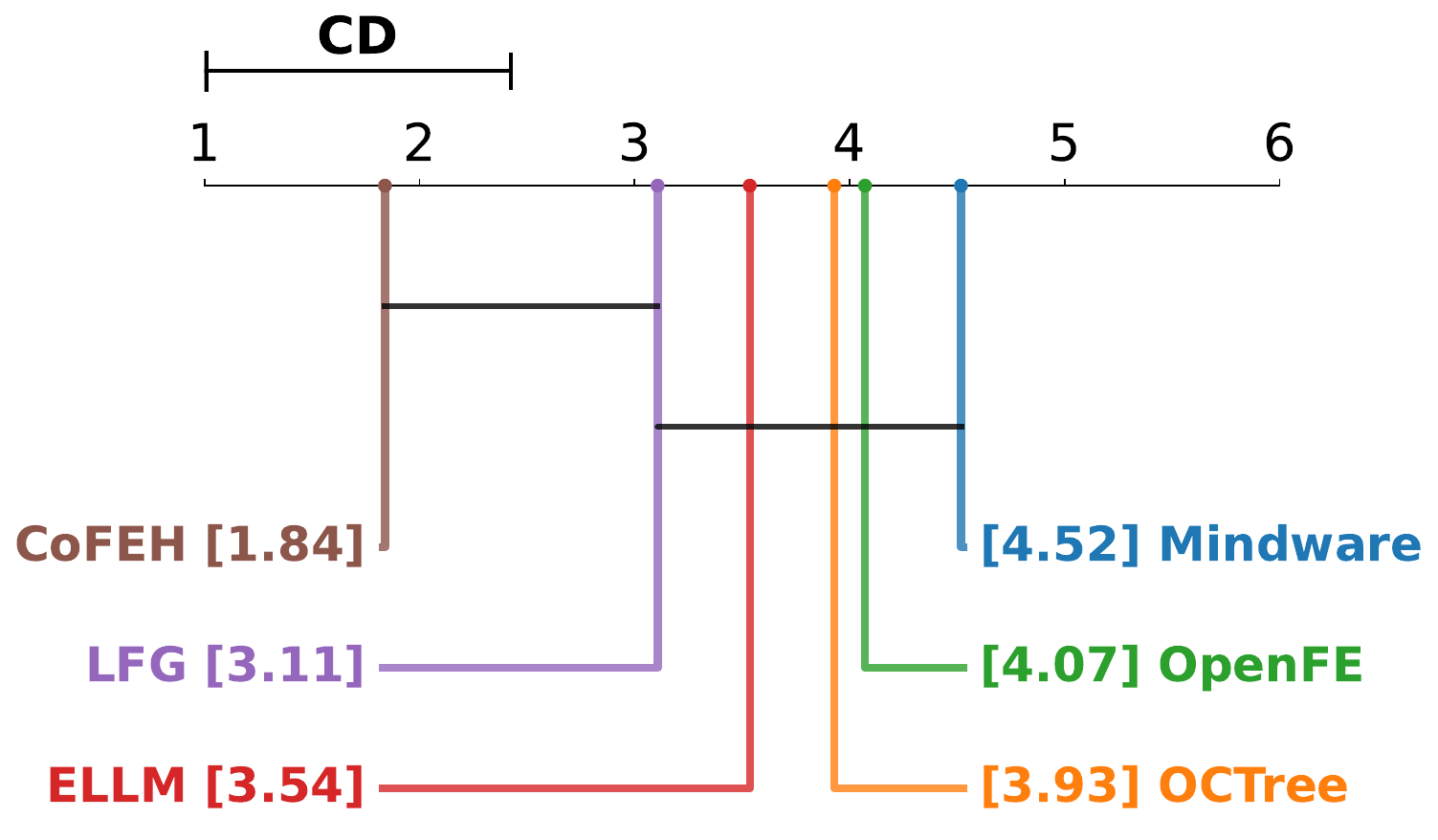}
        \caption{CD plot in standalone FE.}
        \label{fig:cd_fe}
    \end{subfigure}
    % \hfill % 加上这个让两张图撑开，分布在左右两端
    \hspace{5em}
    % 第二个子图
    \begin{subfigure}[b]{0.4\linewidth} % 宽度稍微留一点余量(0.4 -> 0.38)
        \centering
        \includegraphics[width=\linewidth]{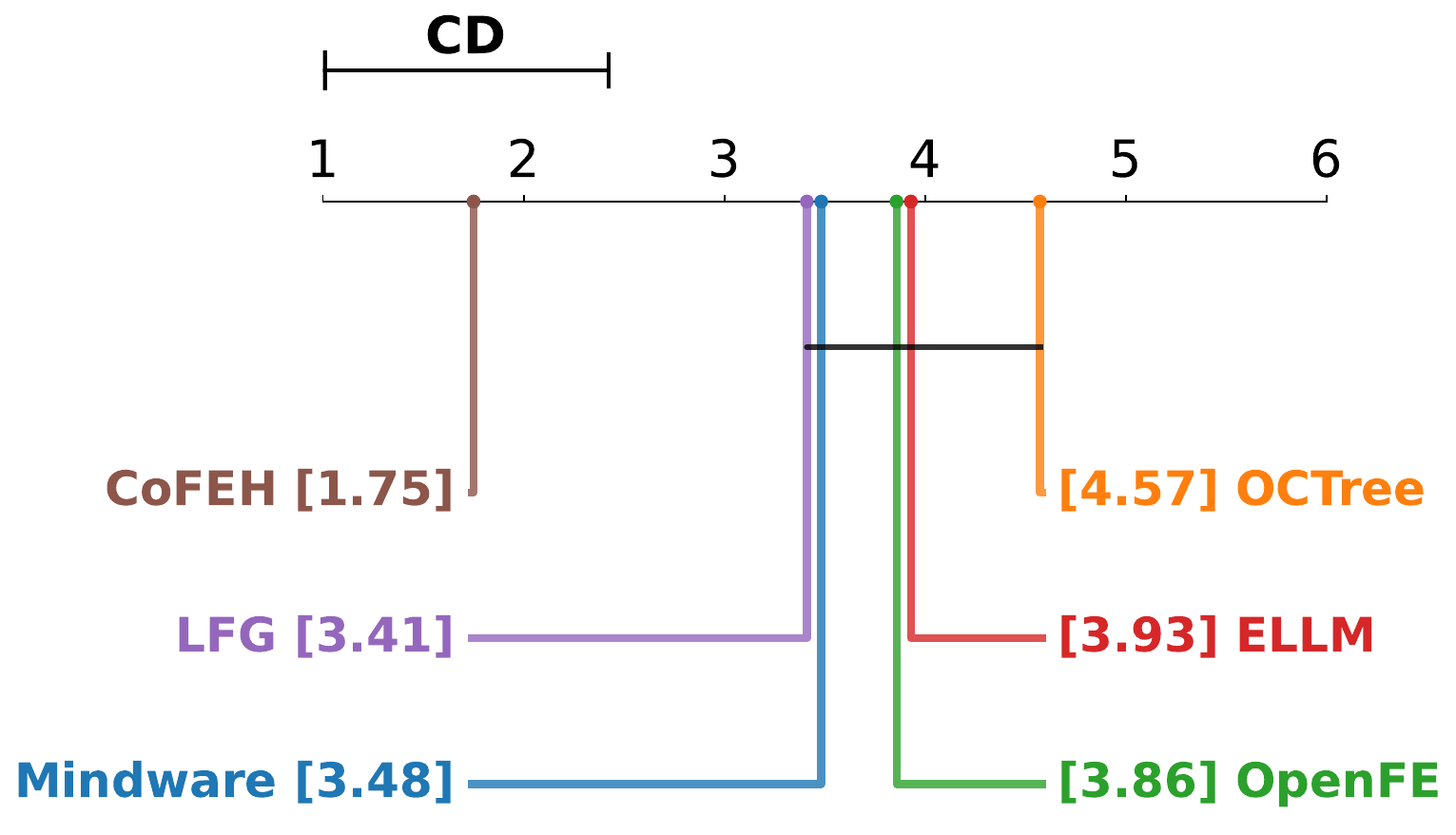}
        \caption{CD plot in joint FE+HPO.}
        \label{fig:cd_fehpo}
    \end{subfigure}
    
    \caption{CD plot of test rank with Nemenyi post-hoc test.}
    \label{fig:cd} % 记得给主图加个 label
\vspace{-1em}
\end{figure*}

\subsection{Statistical Significance Analysis}
\label{app:stat_analysis}

To rigorously evaluate the performance of \sys relative to the baselines, we conduct a statistical significance analysis based on the average ranks across all datasets. As illustrated in Fig.~\ref{fig:cd}, we employ the Friedman test followed by the Nemenyi post-hoc test to generate Critical Difference (CD) diagrams, where methods are connected by a horizontal line if their performance difference is not statistically significant ($\alpha = 0.05$).
We can conclude two observations: 
(i) Standalone FE scenario (Fig.~\ref{fig:cd_fe}): In the pure FE setting, \sys achieves the best average rank of 1.84. The CD plot reveals that \sys is statistically superior to ELLM, OCTree, OpenFE, and Mindware, as their average ranks fall outside the critical difference distance from \sys. While LFG shows competitive performance, \sys maintains a clear numerical lead, establishing its efficacy in FE pipeline optimization.
(ii) Joint FE+HPO scenario (Fig.~\ref{fig:cd_fehpo}): The superiority of \sys becomes even more pronounced in the end-to-end task, reaching an average rank of 1.75. Crucially, in this scenario, \sys is not connected to any baseline by a horizontal bar, indicating that its performance gain is statistically significant against all five competitors. This underscores that our collaborative optimization strategy (interleaved tuning) provides a distinct advantage over both sequential and purely unified baselines.
Overall, Fig.~\ref{fig:cd} provides robust statistical evidence that \sys consistently outperforms existing frameworks. Its ability to maintain significant dominance in the Joint FE+HPO scenario—the most representative of real-world AutoML workflows—confirms its status as a new state-of-the-art for automated pipeline optimization.

\subsection{Effectiveness of Meta-Feature for BO Surrogate Modeling}
\label{app:meta_fit}

In the collaborative optimization process of \sys, the performance $v$ is jointly determined by the FE pipeline state $T$ and the hyperparameter configuration $\boldsymbol{\lambda}$. Since FE pipelines generated by LLMs lack an explicit, continuous search space, we utilize meta-features $\phi(s)$ to vectorize the dataset state after transformations. This allows the surrogate model to map discrete tree nodes into a representative feature space. 
As detailed in Section~\ref{sec:bo_conditioned_on_fe}, we formalize the training dataset for the surrogate model $M$ as:
\begin{equation}
\mathcal{D}_{\text{BO}} = \left\{ \left( [\phi(s_i), \boldsymbol{\lambda}_{i,j}], v_{i,j} \right) \mid s_i \in \mathcal{V}_{\text{tree}}, \forall j \right\},
\end{equation}
where $\mathcal{V}{\text{tree}}$ is the set of nodes in the MCTS tree, and $\boldsymbol{\lambda}_{i,j}, v_{i,j}$ denote the $j$-th ML configuration and its corresponding score evaluated on node $s_i$. We employ Random Forest (RF) as the surrogate model for its efficiency and robust handling of categorical variables.

\noindent\textbf{The role of meta-features}. Without meta-features, the surrogate model is unable to perceive the impact of FE transformations, reducing the input to only hyperparameters $\boldsymbol{\lambda}$. In such a scenario, the training data becomes $\mathcal{D}_{\text{base}} = \{ (\boldsymbol{\lambda}_{i,j}, v_{i,j}) \}$, which introduces significant noise since identical hyperparameter configurations can yield vastly different performances across different FE pipelines.

\noindent\textbf{Evaluation of modeling capability}.
To verify the effectiveness of our joint modeling, we compare the predictive power of the RF surrogate with and without meta-features by collecting the 200 (FE pipeline, ML configuration) pairs generated by \sys in the Joint FE+HPO scenario. For each dataset, we generate out-of-fold (OOF) predictions via 5-fold cross-validation—where the model is trained on four subsets to predict the fifth—to ensure evaluation on unseen configurations and reflect true generalization. The Spearman rank correlation coefficient between the predicted and actual performance is then calculated as the primary metric. This is because the effectiveness of BO depends more on the surrogate's ability to correctly rank potential candidates than on absolute error minimization.

\begin{figure}[t!]
  \centering
  \includegraphics[width=0.9\linewidth]{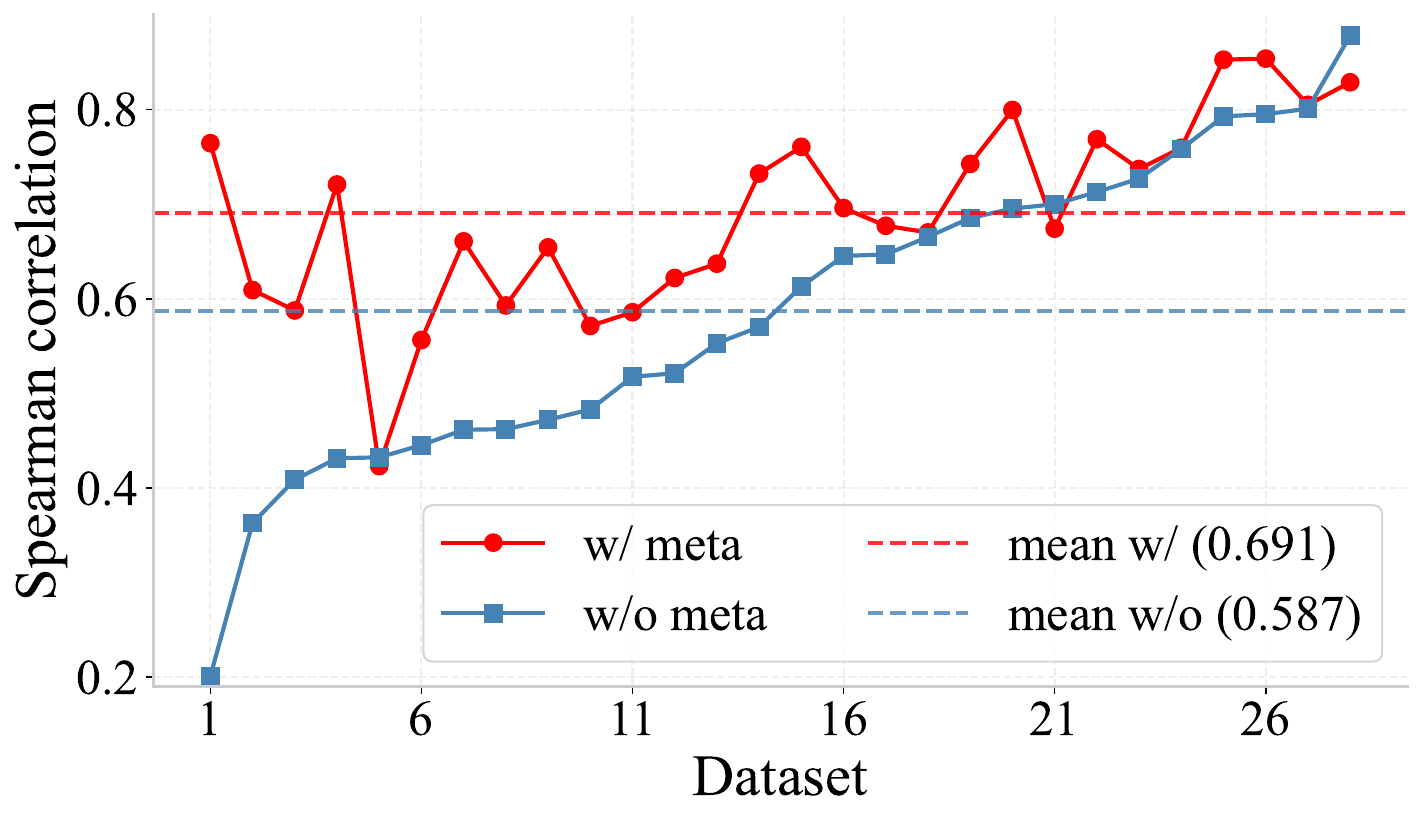}
  \caption{Spearman correlation between surrogate model predictions and true performance.}
  \label{fig:meta_fit}
\end{figure}

The comparative results of surrogate modeling are illustrated in Fig.~\ref{fig:meta_fit}, where datasets are sorted by their Spearman correlation in the "w/o meta-feature" setting. We observe several key trends:
(i) \textbf{Gains in FE-sensitive tasks}: A significant performance gap is evident on datasets where the baseline ($w/o$ meta) fits poorly. For instance, the most extreme case shows a leap from 0.20 to 0.76. This suggests that these tasks are highly sensitive to feature transformations; ignoring the FE pipeline state renders the surrogate model unable to capture the true performance landscape.
(ii) \textbf{Convergence in HPO-sensitive tasks}: On the right side of the plot, where the baseline already achieves high correlation, the improvement from meta-features is relatively marginal. In these scenarios, the model performance is likely dominated by hyperparameter configurations, allowing the surrogate to perform reasonably well even without explicit FE information.
(iii) \textbf{Robustness and potential}: Overall, our joint modeling approach is robust across diverse tasks. It provides substantial gains for FE-sensitive datasets while maintaining or slightly improving performance for HPO-sensitive ones, raising the mean Spearman correlation from 0.587 to 0.691. The few points where both settings yield low correlation may stem from inherent task complexity that exceeds the modeling capacity of the RF surrogate. Notably, as \sys is orthogonal to the specific choice of meta-features, the integration of higher-quality meta-features in the future could further elevate this performance ceiling.
\balance

\begin{figure}[t]
    \centering
    
    % 第一个子图
    \begin{subfigure}[b]{0.49\linewidth} % 宽度稍微留一点余量(0.6 -> 0.58)，防止换行
        \centering
        \includegraphics[width=\linewidth]{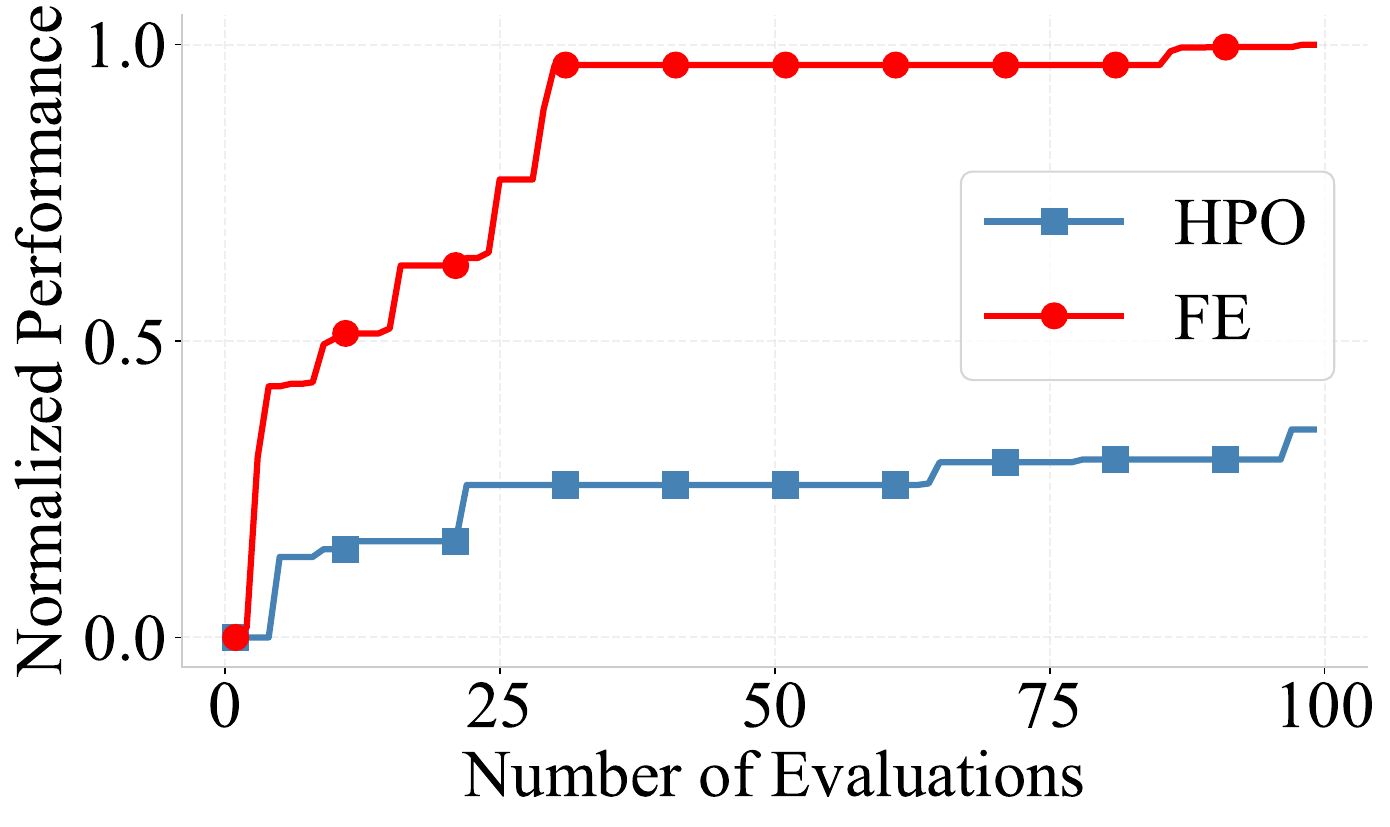}
        \caption{``Diamonds'': FE prop. = 0.7.}
        \label{fig:fe_hpo_curve_diamonds}
    \end{subfigure}
    \hfill % 加上这个让两张图撑开，分布在左右两端
    % \hspace{3em}
    % 第二个子图
    \begin{subfigure}[b]{0.49\linewidth} % 宽度稍微留一点余量(0.4 -> 0.38)
        \centering
        \includegraphics[width=\linewidth]{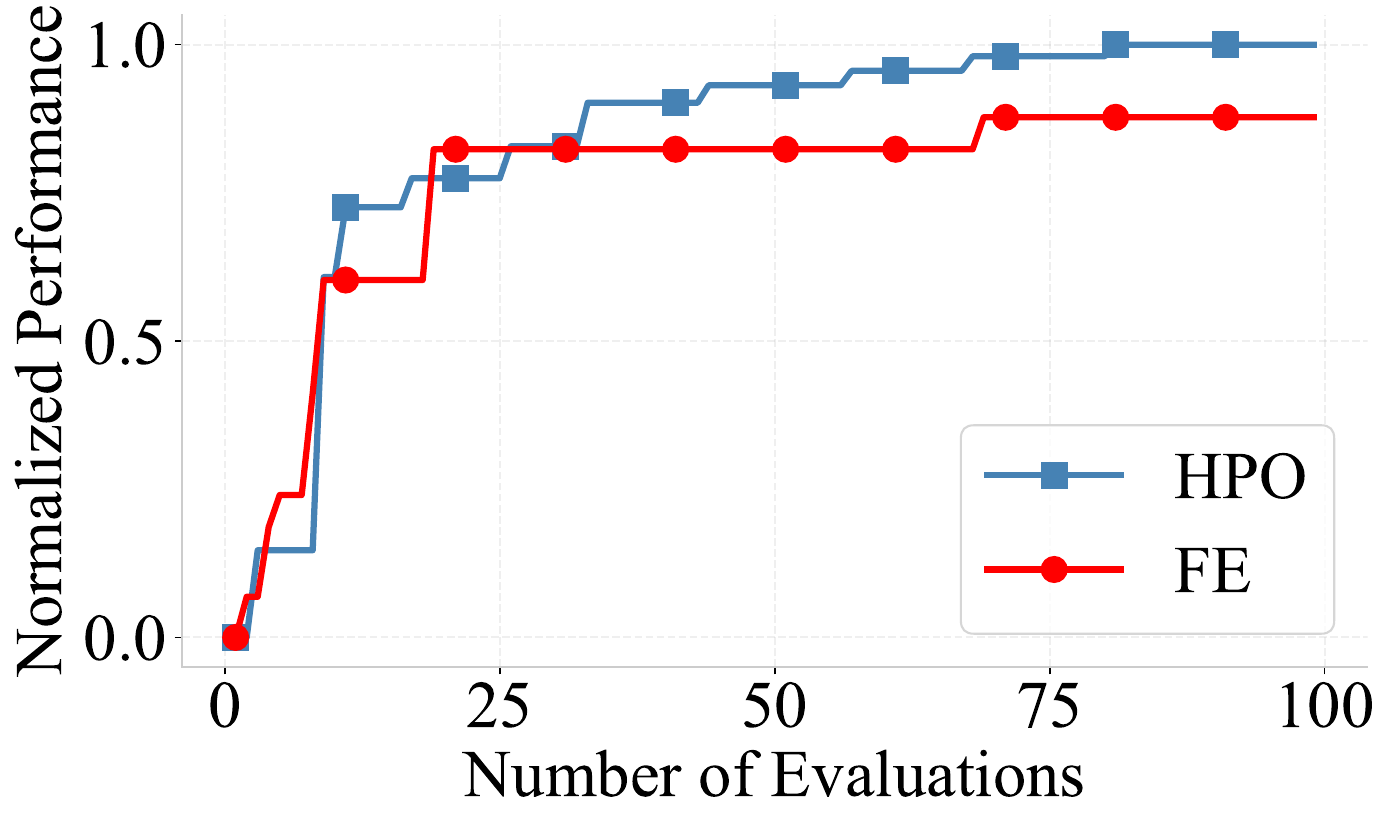}
        \caption{``Jannis'': FE prop. = 0.37.}
        \label{fig:fe_hpo_curve_jannis}
    \end{subfigure}
    
    \vspace{-0.5em}
    \caption{Comparison between standalone FE and HPO.}
    \label{fig:fe_hpo_curve} % 记得给主图加个 label
\end{figure}

\subsection{Case Study of FE vs. HPO}
\label{sec:fe_hpo_case}

To validate the rationality of the dynamic optimizer selector, we examine two representative cases from Fig.~\ref{fig:fe_ratio_hist} exhibiting contrasting resource allocations: one dataset (``diamonds'') with an FE proportion of 0.7 (FE-intensive) and another (``jannis'') with 0.37 (HPO-intensive). To verify if these allocations align with the datasets' inherent sensitivities, we conducted independent 100-iteration runs of pure FE and pure HPO on both tasks.

As illustrated in Fig.~\ref{fig:fe_hpo_curve}, the normalized validation curves reveal distinct performance drivers for each task:
(i) Diamonds (FE-Sensitive): The pure FE trajectory demonstrates a substantially higher performance ceiling and faster convergence compared to pure HPO. This confirms that the primary bottleneck for this task lies in the feature representation, validating the selector's decision to prioritize FE with a 0.7 budget share.
(ii) Jannis (HPO-Sensitive): Conversely, the pure HPO curve surpasses the FE results, indicating that hyperparameter configuration has a more profound impact on accuracy than further feature transformations. The selector correctly identifies this sensitivity, skewing the budget toward the Bayesian HPO module (FE ratio 0.37).
These results empirically substantiate that the dynamic optimizer selector effectively detects the marginal utility of FE versus HPO. By adaptively shifting focus based on the inherent properties of the dataset, \sys ensures the search budget is concentrated where it yields the highest performance gains.

\end{document}